\documentclass[11pt]{article}

\usepackage{hyperref,graphicx,amsmath,amsfonts,amssymb,bm,url,breakurl,epsfig,epsf,color,fullpage,MnSymbol,mathbbol,fmtcount,semtrans,cite,caption,subcaption,multirow,comment}
\usepackage[noend]{algpseudocode}

\usepackage{amssymb}

\usepackage[utf8]{inputenc} 
\usepackage[T1]{fontenc}    
\usepackage{url}            
\usepackage{booktabs}       
\usepackage{amsfonts}       
\usepackage{nicefrac}       
\usepackage{microtype}      

\makeatletter
\providecommand*{\boxast}{%
  \mathbin{
    \mathpalette\@boxit{*}%
  }%
}
\newcommand*{\@boxit}[2]{%
  \sbox0{$\m@th#1\Box$}%
  \ifx#1\displaystyle \ht0=\dimexpr\ht0+.05ex\relax \fi
  \ifx#1\textstyle \ht0=\dimexpr\ht0+.05ex\relax \fi
  \ifx#1\scriptstyle \ht0=\dimexpr\ht0+.04ex\relax \fi
  \ifx#1\scriptscriptstyle \ht0=\dimexpr\ht0+.065ex\relax \fi
  \sbox2{$#1\vcenter{}$}
  \rlap{%
    \hbox to \wd0{%
      \hfill
      \raisebox{%
        \dimexpr.5\dimexpr\ht0+\dp0\relax-\ht2\relax
      }{$\m@th#1#2$}%
      \hfill
    }%
  }%
  \Box
}
\makeatother

  \makeatletter
\def\BState{\State\hskip-\ALG@thistlm}
\makeatother

  \usepackage{mathtools}

\usepackage{titlesec}

\usepackage{tikz}
\usepackage{pgfplots}
\usetikzlibrary{pgfplots.groupplots}

\titleformat{\paragraph}
{\normalfont\normalsize\bfseries}{\theparagraph}{1em}{}
\titlespacing*{\paragraph}
{0pt}{3.25ex plus 1ex minus .2ex}{1.5ex plus .2ex}

\usepackage{movie15}

\usepackage{cite}
\usepackage{caption}
\usepackage[bottom,hang,flushmargin]{footmisc} 

\newcommand{\tsn}[1]{{\left\vert\kern-0.25ex\left\vert\kern-0.25ex\left\vert #1 
    \right\vert\kern-0.25ex\right\vert\kern-0.25ex\right\vert}}

\definecolor{darkred}{RGB}{150,0,0}
\definecolor{darkgreen}{RGB}{0,150,0}
\definecolor{darkblue}{RGB}{0,0,200}
\hypersetup{colorlinks=true, linkcolor=darkred, citecolor=darkgreen, urlcolor=darkblue}

\newtheorem{theorem}{Theorem}[section]

\newtheorem{assumption}{Assumption}

\newtheorem{lemma}[theorem]{Lemma}
\newtheorem{corollary}[theorem]{Corollary}

\newtheorem{definition}[theorem]{Definition}

\newcommand{\eps}{\varepsilon}

\newcommand{\ind}[1]{{\mathcal{I}}(#1)}

\newcommand{\bp}{\beta}
\newcommand{\bn}{\alpha}

\newcommand{\beq}{\begin{equation}}

\newcommand{\eeq}{\end{equation}}

\newcommand{\nn}{\nonumber}
\newcommand{\la}{\lambda}
\newcommand{\A}{{\mtx{A}}}

\newcommand{\B}{{{\mtx{B}}}}

\newcommand{\Sb}{{{\mtx{S}}}}

\newcommand{\Lc}{{\cal{L}}}
\newcommand{\Jc}{{\cal{J}}}

\newcommand{\Qb}{{\mtx{Q}}}
\newcommand{\rng}{\gamma}
\newcommand{\Cb}{{\mtx{C}}}

\newcommand{\Iden}{{\mtx{I}}}

\newcommand{\order}[1]{{\cal{O}}(#1)}

\newcommand{\smn}[1]{{\sigma_{\min}(#1)}}
\newcommand{\smx}[1]{{\sigma_{\max}(#1)}}
\newcommand{\z}{{\vct{z}}}

\newcommand{\tn}[1]{\|{#1}\|_{\ell_2}}
\newcommand{\tin}[1]{\|{#1}\|_{\ell_\infty}}

\newcommand{\bteta}{\boldsymbol{\theta}}

\newcommand{\Nn}{\mathcal{N}}

\newcommand{\vb}{\vct{v}}

\newcommand{\w}{\vct{w}}

\newcommand{\li}{\left<}
\newcommand{\ri}{\right>}

\newcommand{\ab}{\vct{a}}

\newcommand{\g}{{\vct{g}}}

\newcommand{\mat}[1]{{\text{mat}\left(#1\right)}}

\newcommand{\opnorm}[1]{\left\|#1\right\|}
\newcommand{\fronorm}[1]{\left\|#1\right\|_{F}}

\newcommand{\twonorm}[1]{\left\|#1\right\|_{\ell_2}}

\newcommand{\infnorm}[1]{\left\|#1\right\|_{\ell_\infty}}

\newcommand{\abs}[1]{\left|#1\right|}

\newcommand{\x}{\vct{x}}
\newcommand{\rb}{\vct{r}}

\newcommand{\y}{\vct{y}}

\newcommand{\W}{\mtx{W}}
\newcommand{\bgl}{{~\big |~}}

\definecolor{emmanuel}{RGB}{255,127,0}

\newcommand{\R}{\mathbb{R}}
\newcommand{\Pro}{\mathbb{P}}

\newcommand{\E}{\operatorname{\mathbb{E}}}

\newcommand{\grad}[1]{{\nabla\Lc(#1)}}

\newcommand{\vct}[1]{\bm{#1}}
\newcommand{\mtx}[1]{\bm{#1}}

\newcommand{\X}{{\mtx{X}}}

\numberwithin{equation}{section} 

\def \endprf{\hfill {\vrule height6pt width6pt depth0pt}\medskip}
\newenvironment{proof}{\noindent {\bf Proof} }{\endprf\par}

\usepackage{multirow}

\usepackage{algorithm}
\usepackage{algpseudocode}

\title{Towards moderate overparameterization:\\global convergence guarantees for training shallow neural networks}
\author{Samet Oymak\thanks{{Department of Electrical and Computer Engineering, University of California, Riverside, CA}}\quad and\quad Mahdi Soltanolkotabi\thanks{Ming Hsieh Department of Electrical Engineering, University of Southern California, Los Angeles, CA}}
\begin{document}
\maketitle

\begin{abstract} Many modern neural network architectures are trained in an overparameterized regime where the parameters of the model exceed the size of the training dataset. Sufficiently overparameterized neural network architectures in principle have the capacity to fit any set of labels including random noise. However, given the highly nonconvex nature of the training landscape it is not clear what level and kind of overparameterization is required for first order methods to converge to a global optima that perfectly interpolate any labels. A number of recent theoretical works have shown that for very wide neural networks where the number of hidden units is polynomially large in the size of the training data gradient descent starting from a random initialization does indeed converge to a global optima. However, in practice much more moderate levels of overparameterization seems to be sufficient and in many cases overparameterized models seem to perfectly interpolate the training data as soon as the number of parameters exceed the size of the training data by a constant factor. Thus there is a huge gap between the existing theoretical literature and practical experiments. In this paper we take a step towards closing this gap. Focusing on shallow neural nets and smooth activations, we show that (stochastic) gradient descent when initialized at random converges at a geometric rate to a nearby global optima as soon as the square-root of the number of network parameters exceeds the size of the training data. Our results also benefit from a fast convergence rate and continue to hold for non-differentiable activations such as Rectified Linear Units (ReLUs).
\end{abstract}
\section{Introduction}
\subsection{Motivation}
Modern neural networks typically have more parameters than the number of data points used to train them. This property allows neural nets to fit to any labels even those that are randomly generated \cite{zhang2016understanding}. Despite many empirical evidence of this capability the conditions under which this occurs is far from clear. In particular, due to this overparameterization, it is natural to expect the training loss to have numerous global optima that perfectly interpolate the training data. However, given the highly nonconvex nature of the training landscape it is far less clear why (stochastic) gradient descent can converge to such a globally optimal model without getting stock in subpar local optima or stationary points. Furthermore, what is the exact amount and kind of overpametrization that enables such global convergence? Yet another challenge is that due to
overparameterization, the training loss may have infinitely many global minima and it is critical to understand the properties of the solutions found by first-order optimization schemes such as (stochastic) gradient descent starting from different initializations. 

Recently there has been interesting progress aimed at demystifying the global convergence of gradient descent for overparameterized networks. However, most existing results focus on either quadratric activations \cite{soltanolkotabi2018theoretical, venturi2018spurious} or apply to very specialized forms of overparameterization \cite{li2018learning,allen2018learning,allen2018convergence,du2018gradient,zou2018stochastic, du2018gradient2} involving unrealistically wide neural networks where the number of hidden nodes are polynomially large in the size of the dataset. In contrast to this theoretical literature popular neural networks require much more modest amounts of overparameterization and do not typically involve extremely wide architectures. In particular (stochastic) gradient descent starting from a random initialization seems to find globally optimal network parameters that perfectly interpolate the training data as soon as the number of parameters exceed the size of the training data by a constant factor. See Section \ref{numeric} for some numerical experiments corroborating this claim. Also in such overparameterized regimes gradient descent seems to converge much faster than existing results suggest.

In this paper we take a step towards closing the significant gap between the theory and practice of overparameterized neural network training. We show that for training neural networks with one hidden layer, (stochastic) gradient descent starting from a random initialization finds globally optimal weights that perfectly fit any labels as soon as the number of parameters in the model exceed the square of the size of the training data by numerical constants only depending on the input training data. This result holds for networks with differentiable activations. We also develop results of a similar flavor, albeit with slightly worse levels of overparameterization, for neural networks involving Rectified Linear Units (ReLU) activations. Our results also show that gradient descent converges at a much faster rate than existing gurantees. Our theory is based on combining recent results on overparameterized nonlinear learning \cite{Oymak:2018aa} with more intricate tools from random matrix theory and bounds on the spectrum of Hadamard matrices. While in this paper we have focused on shallow neural networks with a quadratic loss, the mathematical techniques we develop are quite general and may apply more broadly. For instance, our techniques may help improve the existing guarantees for overparameterized deep networks (\cite{allen2018convergence,du2018gradient2}) or allow guarantees for other loss functions. We leave a detailed study of these cases to future work.



\subsection{Model}
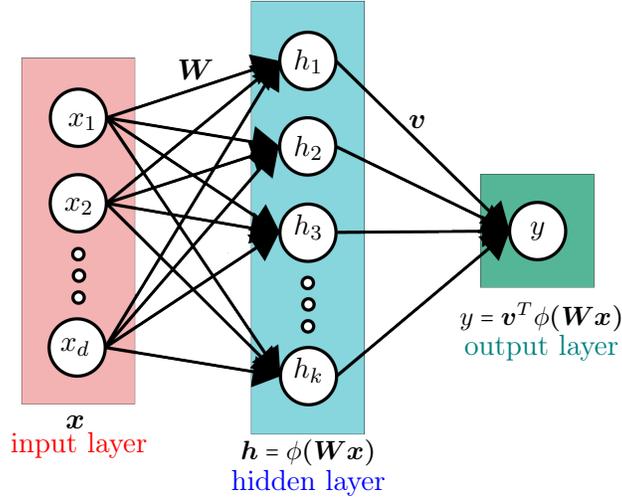
\begin{figure}
\centering
\begin{tikzpicture}
\node at (0,0) {\includegraphics[scale=0.15]{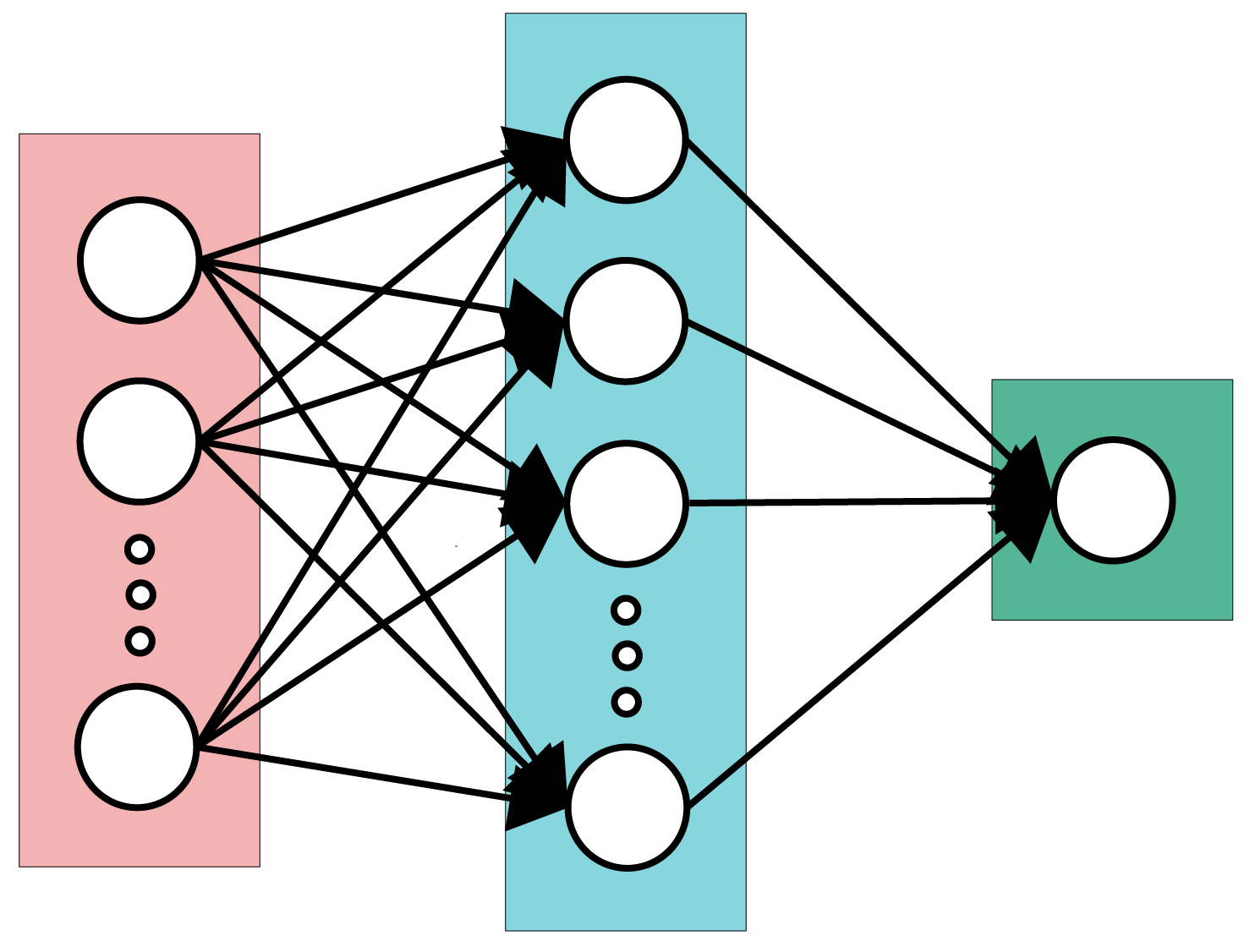}};
\node[black] at (-3,1.3) {$x_1$};
\node[black] at (-3.05,0.17) {$x_2$};
\node[black] at (-3.1,-1.7) {$x_d$};
\node[black] at (-3.1,-2.7) {$\vct{x}$};
\node[red] at (-3.05,-3) {input layer};
\node[black] at (0,2.1) {$h_1$};
\node[black] at (0,0.95) {$h_2$};
\node[black] at (0,-0.1) {$h_3$};
\node[black] at (0,-2) {$h_k$};
\node[black] at (0,-3.1) {{\small$\vct{h}=\phi(\mtx{W}\vct{x})$}};
\node[blue] at (0,-3.5) {hidden layer};
\node[black] at (3.05,-0.15) {$y$};
\node[black] at (1.45,1.3) {$\vct{v}$};
\node[black] at (-1.5,2) {$\mtx{W}$};
\node[black] at (3.1,-1.3) {{\small$y=\vct{v}^T\phi(\mtx{W}\vct{x})$}};
\node[teal] at (3.1,-1.7) {output layer};
\end{tikzpicture}
\caption{Illustration of a one-hidden layer neural net with $d$ inputs, $k$ hidden units and a single output.}
\label{neuralnet}
\vspace{0.5cm}
\end{figure}
We shall focus on neural networks with only one hidden layer with $d$ inputs, $k$ hidden neurons and a single output as depicted in Figure \ref{neuralnet}. The overall input-output relationship of the neural network in this case is a function $f(\cdot;\W):\R^d\rightarrow\R$ that maps the input vector $\vct{x}\in\R^d$ into a scalar output via the following equation
\begin{align*}
\vct{x}\mapsto f(\vct{x};\mtx{W})=\sum_{\ell=1}^k\vct{v}_\ell \phi\left(\langle \vct{w}_\ell,\vct{x}\rangle\right).
\end{align*}
In the above the vectors $\vct{w}_\ell\in\R^d$ contains the weights of the edges connecting the input to the $\ell$th hidden node and $\vct{v}_\ell\in\R$ is the weight of the edge connecting the $\ell$th hidden node to the output. Finally, $\phi:\R\rightarrow\R$ denotes the activation function applied to each hidden node. For more compact notation we gather the weights $\vct{w}_\ell/\vct{v}_\ell$ into larger matrices $\mtx{W}\in\R^{k\times d}$ and $\vct{v}\in\R^k$ of the form
\begin{align*}
\mtx{W}=\begin{bmatrix}\vct{w}_1^T\\\vct{w}_2^T\\\vdots\\\vct{w}_k^T\end{bmatrix}\quad\text{and}\quad\vct{v}=\begin{bmatrix}{v}_1\\{v}_2\\\vdots\\ {v}_k\end{bmatrix}.
\end{align*}
We can now rewrite our input-output model in the more succinct form
\begin{align}
\label{model2}
\vct{x}\mapsto f(\vct{x};\mtx{W}):=\vct{v}^T\phi(\mtx{W}\vct{x}).
\end{align}
Here, we have used the convention that when $\phi$ is applied to a vector it corresponds to applying $\phi$ to each entry of that vector. 
\subsection{Notations}
Before we begin discussing our main results we discuss some notation used throughout the paper. For a matrix $\mtx{X}\in\R^{n\times d}$ we use $\smn{\mtx{X}}$ and $\smx{\mtx{X}}=\opnorm{\mtx{X}}$ to denote the minimum and maximum singular value of $\mtx{X}$. For two matrices 
\begin{align*}
\mtx{A}=\begin{bmatrix}\mtx{A}_1\\\mtx{A}_2\\\vdots\\\mtx{A}_p\end{bmatrix}\in\R^{p\times m}\quad\text{and}\quad \mtx{B}=\begin{bmatrix}\mtx{B}_1\\\mtx{B}_2\\\vdots\\\mtx{B}_p\end{bmatrix}\in\R^{p\times n},
\end{align*}
we define their Khatri-Rao product as $\mtx{A} * \mtx{B}  = [\mtx{A}_1\otimes \mtx{B}_1,\dotsc, \mtx{A}_p\otimes \mtx{B}_p]\in\R^{p\times mn}$, where $\otimes$ denotes the Kronecher product. For two matrices $\A$ and $\mtx{B}$,
we denote their Hadamard (entrywise) product by $\A\odot\mtx{B}$. For a matrix $\mtx{A}\in\R^{n\times n}$, $\mtx{A}^{\odot r}\in\R^{n\times n}$ is defined inductively via $\mtx{A}^{\odot r}=\mtx{A}\odot \left(\mtx{A}^{\odot (r-1)}\right)$ with $\mtx{A}^{\odot 0}=\vct{1}\vct{1}^T$. Similarly, for a matrix $\X\in\R^{n\times d}$ with rows given by $\vct{x}_i\in\R^d$ we define the $r$-way Khatrio-Rao matrix $\X^{*r}\in\R^{n\times d^r}$ as a matrix with rows given by
\begin{align*}
\big[\X^{*r}\big]_i=\left(\underbrace{\vct{x}_i\otimes \vct{x}_i\otimes \ldots \otimes \vct{x}_i}_{r} \right)^T. 
\end{align*}
For a matrix $\mtx{W}\in\R^{k\times d}$ we use vect$(\mtx{W})\in\R^{kd}$ to denote a column vector obtained by concatenating the rows $\vct{w}_1,\vct{w}_2,\ldots,\vct{w}_k\in\R^d$ of $\mtx{W}$. That is,
$\text{vect}(\W)=\begin{bmatrix}\vct{w}_1^T&\vct{w}_2^T &\ldots &\vct{w}_k^T\end{bmatrix}^T$. Similarly, we use $\mat{\vct{w}}\in\R^{k\times d}$ to denote a $k\times d$ matrix obtained by reshaping the vector $\vct{w}\in\R^{kd}$ across its rows. Throughout, for a differentiable function $\phi:\R\mapsto\R$ we use $\phi'$ and $\phi''$ to denote the first and second derivative. If the function is not differentiable but only has isolated non-differentiable points we use $\phi'$ to denote a generalized derivative \cite{Clarke1975}. For instance, for $\phi(z)=ReLU(z)=\max(0,z)$ we have $\phi'(z)=\mathbb{I}_{\{z\ge 0\}}$ with $\mathbb{I}$ denoting the indicator mapping. We use $c$ and $C$ to denote numerical constants whose values may change from line to line. We also use the notation $c_z$ to denote a numerical constant only depending on the variable or function $z$. 
\section{Main results}\label{sec main resu}
When training a neural network, one typically has access to a data set consisting of $n$ feature/label pairs $(\vct{x}_i,y_i)$ with $\vct{x}_i\in\R^d$ representing the features and $y_i$ the associated labels. We wish to infer the best weights $\vct{v},\mtx{W}$ such that the mapping $f(\vct{x};\mtx{W}):=\vct{v}^T\phi(\mtx{W}\vct{x})$ best fits the training data. In this paper we assume $\vct{v}\in\R^k$ is fixed and we train for the input-to-hidden weights $\mtx{W}$ via a quadratic loss. The training optimization problem then takes the form
\begin{align}
\label{neuralopt}
\underset{\mtx{W}\in\R^{k\times d}}{\min}\text{ }\mathcal{L}(\mtx{W}):=\frac{1}{2}\sum_{i=1}^n \left(\vct{v}^T\phi\left(\mtx{W}\vct{x}_i\right)-y_i\right)^2.
\end{align}
To optimize this loss we run (stochastic) gradient descent starting from a random initialization $\W_0$. We wish to understand: (1) when such iterative updates lead to a globally optimal solution that perfectly interpolates the training data, (2) what are the properties of the solutions these algorithms converge to, and (3) what is the required amount of overparameterization necessary for such events to occur. We begin by stating results for training via gradient descent for smooth activations in Section \ref{smooth} followed by ReLU activations in Section \ref{ReLU}. Finally, we discuss results for training via Stochastic Gradient Descent (SGD) in Section \ref{SGD}.


\subsection{Training networks with smooth activations via gradient descent}
\label{smooth}
In our first result we consider a one-hidden layer neural network with smooth activations and study the behavior of gradient descent in an over-parameterized regime where the number of parameters is sufficiently large. 
\begin{theorem}\label{thmshallowsmooth} Consider a data set of input/label pairs $\vct{x}_i\in\R^d$ and $y_i\in\R$ for $i=1,2,\ldots,n$ aggregated as rows/entries of a data matrix $\mtx{X}\in\R^{n\times d}$ and a label vector $\vct{y}\in\R^n$. Without loss of generality we assume the dataset is normalized so that $\twonorm{\vct{x}_i}=1$. Also consider a one-hidden layer neural network with $k$ hidden units and one output of the form $\vct{x}\mapsto \vct{v}^T\phi\left(\mtx{W}\vct{x}\right)$ with $\mtx{W}\in\R^{k\times d}$ and $\vct{v}\in\R^k$ the input-to-hidden and hidden-to-output weights. We assume the activation $\phi$ has bounded derivatives i.e.~$\abs{\phi'(z)}\le B$ and $\abs{\phi''(z)}\le B$ for all $z$ and set $\mu_\phi=\E_{g\sim\mathcal{N}(0,1)}[g\phi'(g)]$. Furthermore, we set half of the entries of $\vct{v}$ to $\frac{\twonorm{\y}}{\sqrt{kn}}$ and the other half to $-\frac{\twonorm{\y}}{\sqrt{kn}}$\footnote{If $k$ is odd we set one entry to zero $\lfloor \frac{k-1}{2}\rfloor$ to $\frac{\twonorm{\y}}{\sqrt{kn}}$ and $\lfloor \frac{k-1}{2}\rfloor$ entries to $-\frac{\twonorm{\y}}{\sqrt{kn}}$.} and train only over $\mtx{W}$. Starting from an initial weight matrix $\mtx{W}_0$ selected at random with i.i.d.~$\mathcal{N}(0,1)$ entries we run Gradient Descent (GD) updates of the form $\mtx{W}_{\tau+1}=\mtx{W}_\tau-\eta\nabla \mathcal{L}(\mtx{W}_\tau)$ on the loss \eqref{neuralopt} with step size $\eta= \frac{n\bar{\eta}}{2B^2\tn{\vct{y}}^2\opnorm{\X}^2}$ where $\bar{\eta}\le 1$. Then, as long as
\begin{align}
\label{overparam}
\sqrt{kd}\ge c \frac{B^2}{\mu_\phi^2}(1+\delta) \kappa(\mtx{X}) n\quad\text{holds with}\quad \kappa(\mtx{X}):=\frac{\sqrt{\frac{d}{n}}\opnorm{\mtx{X}}}{\sigma_{\min}^2\left(\mtx{X}*\mtx{X}\right)},
\end{align}
and $c$ is a fixed numerical constant, then with probability at least $1-\frac{1}{n}-e^{-\delta^2\frac{n}{2\opnorm{\X}^2}}$ all GD iterates obey
\begin{align*}
&\tn{f(\mtx{W}_\tau)-\y}\leq \left(1-\frac{\bar{\eta}}{32} \frac{\mu_{\phi}^2}{B^2}\frac{\sigma_{\min}^2\left(\X*\X\right)}{\opnorm{\X}^2}\right)^\tau\tn{f(\mtx{W}_0)-\y},\\
&\frac{\mu_{\phi}}{\sqrt{32}}\frac{\twonorm{\y}}{\sqrt{n}}\smn{\X*\X}\fronorm{\mtx{W}_\tau-\mtx{W}_0}+\tn{f(\mtx{W}_\tau)-\y}\leq \tn{f(\mtx{W}_0)-\y}.
\end{align*}
\noindent Furthermore, the total gradient path obeys
\begin{align*}
\sum_{\tau=0}^{\infty} \fronorm{\W_{\tau+1}-\W_\tau}\le \frac{\sqrt{32}}{\mu_{\phi}}\frac{\sqrt{n}}{\twonorm{\y}}\frac{\tn{f(\mtx{W}_0)-\y}}{\smn{\X*\X}}.
\end{align*}
\end{theorem}
We would like to note that we have chosen to state our results based on easy to calculate quantities such as $\sigma_{\min}^2\left(\X *\X\right)$ and $\mu_\phi$. As it becomes clear in the proofs a more general result holds where the theorem above and its conclusions can be stated with $\mu_\phi\sigma_{\min}\left(\mtx{X}*\mtx{X}\right)$ replaced with a quantity that only depends on the expected minimum singular value of the Jacobian of the neural network mapping at the random initialization (See Theorem \ref{thmshallowsmoothg} in the proofs for details). Using this more general result combined with well known calculations involving Hermite polynomials one can develop other interpretable results. For instant we can show that $\sigma_{\min}(\X*\X)$ can be replaced with higher order Khatrio-Rao products (i.e.~$\sigma_{\min}(\X^{*r})$). 

Before we start discussing the conclusions of this theorem let us briefly discuss the scaling of various quantities. When $n\gtrsim d$, in many cases we expect $\opnorm{\X}$ to grow with $\sqrt{n/d}$ and $\sigma_{\min}(\X*\X)$ to be roughly a constant so that $\kappa(\X)$ is typically a constant (see Corollary \ref{maincor2} below for a precise statement). Thus based on \eqref{overparam} the typical scaling required in our results is $kd \gtrsim n^2$. That is, the conclusions of Theorem \ref{thmshallowsmooth} holds with high probability as soon as the square of the number of parameters of the model exceed the number of training data by a fixed numerical constant. To the extent of our knowledge this result is the first of its kind only requiring the number of parameters to be sufficiently large w.r.t.~the training data rather than the number of hidden units w.r.t.~the size of the training data. That said, as we demonstrate in Section \ref{numeric} neural networks seem to work with even more modest amounts of overparameterization and when the number of parameters exceed the size of the training data by a numerical constant i.e.~$kd\gtrsim n$. We hope to close this remaining gap in future work. We also note that based on this typical scaling the convergence rate is on the order of $\left(1-c\frac{d}{n}\right)$. 

We briefly pause to also discuss the case where one assumes $n\lesssim d$ (although this is not a typical regime of operation in neural networks). In this case both $\opnorm{\X}$ and $\sigma_{\min}(\X*\X)$ are of the order of one and thus $\kappa$ scales as $\sqrt{d/n}$. Thus, the overparmeterization requirement \eqref{overparam} reduces to $k\gtrsim n$. Thus, in this regime we can perfectly fit any labels as soon as the number of hidden units exceeds the size of the training data. We also note that in this regime the convergence rate is a fixed numerical constant independent of any of the dimensions.

Before we start discussing the conclusions of this theorem let us state a simple corollary that clearly illustrates the scaling discussed above for randomly generated input data. The proof of this simple corollary is deferred to Appendix \ref{corpf}.
\begin{corollary}\label{maincor2} Consider the setting of Theorem \ref{thmshallowsmooth} above with $\eta=\frac{n}{2B^2\tn{\vct{y}}^2\opnorm{\X}^2}$. Furthermore, assume the we use the softplus activation $\phi(z)=\log(1+e^z)$ and the input data points $\vct{x}_1, \vct{x}_2, \ldots, \vct{x}_n$ are generated i.i.d.~uniformly at random from the unit sphere of $\R^d$ where $d\le n\le cd^2$. Then, 
as long as
\begin{align}
\label{overparamsp}
\sqrt{kd}\ge C n,
\end{align}
with probability at least $1-\frac{2}{n}-e^{-\frac{d}{4}}-ne^{-\gamma_1\sqrt{n}}-(2n+1)e^{-\gamma_2 d}$ all GD iterates obey
\begin{align}
\label{firstsmooth}
&\tn{f(\mtx{W}_\tau)-\y}\leq 3\left(1-c_1\frac{d}{n} \right)^\tau\tn{\y},\\
&\fronorm{\mtx{W}_\tau-\mtx{W}_0}+c_2\frac{\sqrt{n}}{\twonorm{\y}}\tn{f(\mtx{W}_\tau)-\y}\leq c_3\sqrt{n}.
\label{secsmooth}
\end{align}
Here, $\gamma_1, \gamma_2, c,C, c_1, c_2$, and $c_3$ are fixed numerical constants. 
\end{corollary}
We would like to note that while for simplicity this corollary is stated for data points that are uniform on the unit sphere, as it becomes clear in the proof, this result continues to hold for a variety of other generic\footnote{Informally, we call a set of points generic as long as no subset of them belong to an algebraic manifold.} data models with the same scaling.
The corollary above clarifies that the typical scaling required in our results is indeed $kd \gtrsim n^2$. That is, the conclusions of Theorem \ref{thmshallowsmooth} holds with high probability as soon as the square of the number of parameters of the model exceed the number of training data by a fixed numerical constant. 


The theorem and corollary above show that under $kd \gtrsim n^2$ overparameterization Gradient Descent (GD) iterates have a few interesting properties properties:

\noindent\textbf{Zero traning error:} The first property demonstrated by Theorem \ref{thmshallowsmooth} above is that the iterates converge to a global optima. This holds despite the fact that the fitting problem may be highly nonconvex in general. Indeed, based on \eqref{firstsmooth} the fitting/training error $\twonorm{f(\W_\tau)-\vct{y}}$ achieved by Gradient Descent (GD) iterates converges to zero. Therefore, GD can perfectly interpolate the data and achieve zero training error. Furthermore, the algorithm enjoys a fast geometric rate of convergence to this global optima. In particular to achieve a relative accuracy of $\epsilon$ (i.e.~$\twonorm{f(\W_\tau)-\y}/\twonorm{\y}\le \epsilon$) the required number of iterations $\tau$ is of the order of $\tau \gtrsim \frac{n}{d}\log(1/\epsilon)$.

\noindent\textbf{Gradient descent iterates remain close to the initialization:}  The second interesting aspect of our result is that we guarantee the GD iterates never leave a neighborhood of radius of the order of $\sqrt{n}$ around the initial point. That is the GD iterates remain rather close to the initialization.\footnote{Note that $\fronorm{\W_0}\approx \sqrt{kd}>>\sqrt{n}$ so that this radius is indeed small.}
Furthermore, \eqref{secsmooth} shows that for all iterates the weighted sum of the distance to the initialization and the misfit error remains bounded so that as the loss decreases the distance to the initialization only moderately increases.

\noindent\textbf{Gradient descent follows a short path:} Another interesting aspect of the above results is that the total length of the path taken by gradient descent remains bounded and is of the order of $\sqrt{n}$.

\subsection{Training ReLU networks via gradient descent}
\label{ReLU}
The results in the previous section focused on smooth activations and therefore does not apply to non-differentiable activations and in particular the widely popular ReLU activations. In the next theorem we show that a similar result continues to hold when ReLU activations are used.
\begin{theorem}\label{reluthm} Consider the setting of Theorem \ref{thmshallowsmooth} with the activations equal to $\phi(z)=ReLU(z):=\max(0,z)$ and the step size $\eta= \frac{n}{3\tn{\y}^2\|\X\|^2}\bar{\eta}$ with $\bar{\eta}\le 1$. Then, as long as
\begin{align}
\sqrt{kd}\ge C(1+\delta) \frac{n^2}{d} \kappa^3\left(\X\right)\sigma_{\min}^2\left(\X*\X\right)\quad\text{holds with}\quad \kappa(\mtx{X}):=\frac{\sqrt{\frac{d}{n}}\opnorm{\mtx{X}}}{\sigma_{\min}^2\left(\mtx{X}*\mtx{X}\right)},\label{k big}
\end{align}
and $\gamma$ and $c$ fixed numerical constants, then with probability at least $1-\frac{1}{n}-e^{-\delta^2\frac{n}{\opnorm{\X}^2}}-ne^{-n}$ all GD iterates obey
\begin{align*}
&\tn{f(\mtx{W}_\tau)-\y}\leq \left(1-\frac{\bar{\eta}}{48\pi}\frac{\sigma_{\min}^2\left(\X*\X\right)}{\opnorm{\X}^2}\right)^\tau\tn{f(\mtx{W}_0)-\y},\\
&\frac{1}{12\sqrt{\pi}}\frac{\twonorm{\y}}{\sqrt{n}}\smn{\X*\X}\fronorm{\mtx{W}_\tau-\mtx{W}_0}+\tn{f(\mtx{W}_\tau)-\y}\leq \tn{f(\mtx{W}_0)-\y}.
\end{align*}
\end{theorem}
Also similar to Corollary \ref{maincor2} we can state the following simple corollary to better understand the requirement in typical instances.
\begin{corollary}\label{maincorrel} Consider the setting of Theorem \ref{reluthm} above with $\eta=\frac{n}{\tn{\y}^2\|\X\|^2}$. Furthermore, assume the input data points $\vct{x}_1, \vct{x}_2, \ldots, \vct{x}_n$ are generated i.i.d.~uniformly at random from the unit sphere of $\R^d$ where $d\le n\le cd^2$. Then, 
as long as
\begin{align}
\label{overparamrl}
\sqrt{kd}\ge C \frac{n^2}{d},
\end{align}
with probability at least $1-\frac{2}{n}-e^{-{d}}-ne^{-\gamma_1\sqrt{n}}-(2n+1)e^{-\gamma_2 d}$ all GD iterates obey
\begin{align*}
&\tn{f(\mtx{W}_\tau)-\y}\leq 3\left(1-c_1\frac{d}{n} \right)^\tau\tn{\y},\\
&\fronorm{\mtx{W}_\tau-\mtx{W}_0}+c_2\frac{\sqrt{n}}{\twonorm{\y}}\tn{f(\mtx{W}_\tau)-\y}\leq c_3\sqrt{n}.
\end{align*}
Here, $\gamma_1, \gamma_2, c,C, c_1, c_2$, and $c_3$ are fixed numerical constants. 
\end{corollary}
The theorem and corollary above show that all the nice properties of GD with smooth activations continue to hold for ReLU activations. The only difference is that the required overparameterization is now of the form $\sqrt{kd}\ge C \frac{n^2}{d}$ which is suboptimal compared to the smooth case by a factor of $n/d$.


Our discussion so far focused on results based on the minimum singular value of the second order Khatrio-Rao product $\X*\X$ or higher order products $\X^{*r}$. The reason we require these minimum singular values to be positive is to ensure diversity in the data set. Indeed, if two data points are the same but have different output labels there is no way of achieving zero training error. However, assuming these minimum singular values are positive is not the only way to ensure diversity and our results apply more generally (see Theorem \ref{reluthmg} in the proofs). Another related and intuitive criteria for ensuring diversity is assuming the input samples are sufficiently separated as defined below.
\begin{assumption} [$\delta$-separable data] \label{ass sep}Let $\delta>0$ be a scalar. Consider a data set consisting of $n$ samples $\vct{x}_1,\vct{x}_2,\ldots,\vct{x}_n\in\R^d$ all with unit Euclidian norm. We assume that any pair of points $\x_i$ and $\x_j$ obey
\[
\min(\tn{\x_i-\x_j},\tn{\x_i+\x_j})\geq \delta.
\]
\end{assumption}
We now state a result based on this minimum separation assumption. This result is a corollary of our meta theorem (Theorem \ref{reluthmg}) discussed in the proofs.
\begin{theorem}\label{reluthmsep} Consider the setting of Theorem \ref{thmshallowsmooth} with the activations equal to $\phi(z)=ReLU(z):=\max(0,z)$ and the step size $\eta= \frac{n}{3\tn{\y}^2\|\X\|^2}\bar{\eta}$ with $\bar{\eta}\le 1$. Suppose Assumption \ref{ass sep} holds for some $\delta>0$ and let $c,C>0$ be two numerical constants. Suppose number of hidden nodes satisfy
\begin{align}
k\geq C(1+\nu)^2\frac{n^9\|\X\|^6}{\delta^4},
\end{align}
Then with probability at least $1-\frac{2}{n}-e^{-\nu^2\frac{n}{\opnorm{\X}^2}}$ all GD iterates obey
\begin{align*}
\tn{f(\mtx{W}_\tau)-\y}\leq \left(1-c\frac{\bar{\eta}\delta }{n^2\opnorm{\X}^2}\right)^\tau\tn{f(\mtx{W}_0)-\y},\\
\end{align*}
\end{theorem}
We would like to note that related works \cite{li2018learning,allen2018convergence,zou2018stochastic} consider slight variations of this assumption for training ReLU networks to give overparameterized learning guarantees where the number of hidden nodes grow polynomially in $n$.  Our results seem to have much better dependencies on $n$ compared to these works. Furthermore, we do not require the number of hidden nodes to scale with the desired training accuracy ($\mathcal{L}(\W)\le \epsilon$) as required by \cite{li2018learning}.

\subsection{Training using SGD}
\label{SGD}
The most widely used algorithm for training neural networks is Stochastic Gradient Descent (SGD) and its variants. A natural implementation of SGD is to sample a data point at random and use that data point for the gradient updates. Specifically, let $\{\gamma_\tau\}_{\tau=0}^\infty$ be an i.i.d.~sequence of integers chosen uniformly from $\{1,2,\ldots,n\}$, the SGD iterates take the form
\begin{align}
\W_{\tau+1}=\W_\tau+\eta (y_{\rng_\tau}-f(\x_{\rng_\tau};\W_\tau))\nabla f(\x_{\rng_\tau};\W_\tau).\label{sgd eq}
\end{align}
Here, $G(\vct{\theta}_\tau;\gamma_\tau)$ is the gradient on the $\gamma_\tau$th training sample. We are interested in understanding the trajectory of SGD for neural network training e.g.~the required overparameterization and the associated rate of convergence. We state our result for smooth activations. An analogous result also holds for ReLU activations but we omit the statement to avoid repetition.
\begin{theorem}\label{SGDthmnn} Consider the setting and assumptions of Theorem \ref{thmshallowsmooth} where we use the SGD updates \eqref{sgd eq} in lieu of GD updates with a step size $\eta=\frac{\mu^2(\phi)}{9\nu B^4}\frac{n}{\twonorm{\y}^2} \frac{\sigma_{\min}^2(\X*\X)}{\opnorm{\X}^2}\bar{\eta}$ with $\bar{\eta}\le 1$ and $\nu\ge 3$. 
Set initial weights $\W_0$ with i.i.d.~$\mathcal{N}(0,1)$ entries. Then, with probability at least $1-\frac{1}{n}-e^{-\delta^2\frac{n}{2\opnorm{\X}^2}}$ over $\W_0$, there exists an event $E$\footnote{This event is over the randomness introduced by the SGD updates given fixed $\W_0$.}  which holds with probability at least $\mathbb{P}(E)\ge1-\frac{4}{\nu}\left(\frac{3B\|\X\|}{\mu_\phi\smn{\X*\X}}\right)^{\frac{1}{kd}}$ such that, starting from $\W_0$ and running stochastic gradient descent updates of the form \eqref{sgd eq}, all iterates obey
\begin{align}
\E\Big[\twonorm{f(\W_\tau)-\vct{y}}^2\mathbb{1}_{E}\Big]\le&\left(1-\frac{\bar{\eta}}{144n}\frac{\mu^4(\phi)}{\nu B^4}\frac{\sigma_{\min}^4(\X*\X)}{\opnorm{\X}^2}\right)^\tau\twonorm{f(\W_0)-\vct{y}}^2,\label{sgerr}
\end{align}
Furthermore,  on this event the SGD iterates never leave the local neighborhood $\fronorm{\mtx{W}_\tau-\mtx{W}_0}\le c\nu\sqrt{n}$ with $c$ a fixed numerical constant.
\end{theorem}

This result shows that SGD converges to a global optima that is close to the initialization. Furthermore, SGD always remains in close proximity to the initialization with high probability. To assess the rate of convergence, let us assume generic data and $n\geq d$, so that we have $\|\X\|\sim\sqrt{n/d}$ and $\sigma_{\min}(\X*\X)$ scales as a constant. Then, the result above shows that to achieve a relative accuracy of $\eps$ the number of SGD iterates required is of the order of $\tau\gtrsim\frac{n^2}{d}\log(\frac{1}{\eps})$. This is essentially on par with our earlier result on gradient descent by noting that $n$ SGD iterations require similar computational effort to one full gradient with both approaches requiring $\frac{n}{d}\log(1/\epsilon)$ passes through the data.

\section{The need for overparameterization beyond width}
In this section we would like to further clarify why understanding overparameterization beyond width is particularly important. To see this, we shall set the input-to-hidden weights at random (as used for initialization) and consider the optimization over the output layer weights $\vct{v}\in\R^k$. This optimization problem has the form
\begin{align}
\label{vopt}
\mathcal{L}(\vct{v}):=\frac{1}{2}\sum_{i=1}^n \left(\vct{v}^T\phi\left(\W\vct{x}_i\right)-\y_i\right)^2=\frac{1}{2}\twonorm{\phi\left(\X\W^T\right)\vct{v}-\y}^2,
\end{align}
which is a simple least-squares problem with a globally optimal solution given by 
\begin{align*}
\hat{\vct{v}}:=\mtx{\Phi}^T\left(\mtx{\Phi}\mtx{\Phi}^T\right)^{-1}\y\quad\text{where}\quad \mtx{\Phi}:=\phi\left(\mtx{X}\W^T\right).
\end{align*}
This simple observation shows that the simple least-squares optimization over the output weights achieves zero training as soon as $\mtx{\Phi}$ has full column rank. Thus, in such a setting a simple kernel regression using the random features $\phi(\W\vct{x}_1), \phi(\W\vct{x}_2), \ldots ,\phi(\W\vct{x}_n)$ suffices to perfectly interpolate the data. In this section we wish to understand the amount and kind of overparameterization where such a simple strategy suffices. We thus need to understand the conditions under which the matrix $\phi\left(\X\W^T\right)$ has full row rank. To make things quantitative we need the following definition.
\begin{definition}[Output feature covariance and eigenvalue] We define the output feature covariance matrix as
\begin{align*}
\widetilde{\mtx{\Sigma}}(\X)=\E_{\vct{w}}\big[\phi\left(\X\vct{w}\right)\phi\left(\X\vct{w}\right)^T\big],
\end{align*}
where $\vct{w}\in\R^d$ has a $\mathcal{N}(\vct{0},\mtx{I}_d)$ distribution. We use $\widetilde{\lambda}(\X)$ to denote the corresponding minimum eigenvalue i.e.~$\widetilde{\lambda}(\X)=\lambda_{\min}\left(\widetilde{\Sigma}(\X)\right)$.
\end{definition}
With this definition in place we are now ready to state the main result of this section.
\begin{theorem}\label{thm sense} Consider a data set of input/label pairs $\vct{x}_i\in\R^d$ and $y_i\in\R$ for $i=1,2,\ldots,n$ aggregated as rows/entries of a data matrix $\mtx{X}\in\R^{n\times d}$ and a label vector $\vct{y}\in\R^n$. Without loss of generality we assume the dataset is normalized so that $\twonorm{\vct{x}_i}=1$. Also consider a one-hidden layer neural network with $k$ hidden units and one output of the form $\vct{x}\mapsto \vct{v}^T\phi\left(\mtx{W}\vct{x}\right)$ with $\mtx{W}\in\R^{k\times d}$ and $\vct{v}\in\R^k$ the input-to-hidden and hidden-to-output weights. We assume the activation $\phi$ is bounded at zero i.e.~$\abs{\phi(0)}\le B$ and has a bounded derivative i.e.~$\abs{\phi'(z)}\le B$ for all $z$. We set $\mtx{W}$ to be a random matrix with i.i.d.~$\mathcal{N}(0,1)$ entires. Also assume
\begin{align*}
k\ge C \log^2(n)\frac{n}{\widetilde{\lambda}(\X)}.
\end{align*}
Then, the matrix $\mtx{\Phi}:=\phi\left(\X\mtx{W}^T\right)$ has full row rank with the minimum eigenvalue obeying
\begin{align*}
\lambda_{\min}\left(\mtx{\Phi}\mtx{\Phi}^T\right)\ge \frac{1}{2}k\left(\widetilde{\lambda}(\X)-\frac{6B}{n^{100}}\right).
\end{align*}
Thus, the global optima of \eqref{vopt} achieves zero training error as long as $\widetilde{\lambda}(\X)\ge \frac{6B}{n^{100}}$.
\end{theorem} 
%
%
%
We note that one can develop interpretable lower bounds for $\widetilde{\lambda}$ (see Appendix \ref{hermitesec}). For instance, in Appendix \ref{hermitesec} we show that
\begin{align*}
\widetilde{\lambda}(\X)\ge\gamma_\phi^2\sigma_{\min}^2\left(\X*\X\right)\quad\text{with}\quad \gamma_{\phi}=\frac{1}{\sqrt{2}}\E[\phi(g)(g^2-1)].
\end{align*}
As we discussed in the previous sections for generic or random data $\sigma_{\min}^2\left(\X*\X\right)$ often scales like a constant. In turn, based on the above inequality $\widetilde{\lambda}(\X)$ also scales like a constant. Thus, the above theorem shows that as long as the neural network is wide enough in the sense that $k\gtrsim n$, with high probability on can achieve perfect interpolation and the global optima by simply fitting the last layer with the input-to-hidden weights set randomly. Of course the optimization problem over $\W$ is significantly more challenging to analyze (the setting in this paper and other publications \cite{li2018learning,allen2018convergence,du2018gradient,zou2018stochastic, du2018gradient2}). However, this simple baseline result suggests that there is no fundamental barrier to understanding perfect interpolation for $k\gtrsim n$ wide networks. In particular, as discussed earlier the result above can be thought of as kernel learning with random features. Indeed, in this settings one can also show the solutions found by (stochastic) gradient descent converges to the least-norm solution and does indeed generalize. Furthermore, neural networks are often trained with the number of hidden nodes of the at each intermediate layer significantly smaller than the data size. Thus to truly understand the behavior of neural network training and demystify their success beyond kernel learning it is crucially important to focus on \emph{moderately} overparameterized networks where the number of data points is only moderately larger than the number of parameters used for training. We hope the discussion above can help focus future theoretical investigations to this moderately overparameterized regime.

\section{Numerical experiments}
\label{numeric}
\begin{figure}
    \centering
    \begin{subfigure}[b]{0.49\textwidth}
        \includegraphics[width=\textwidth]{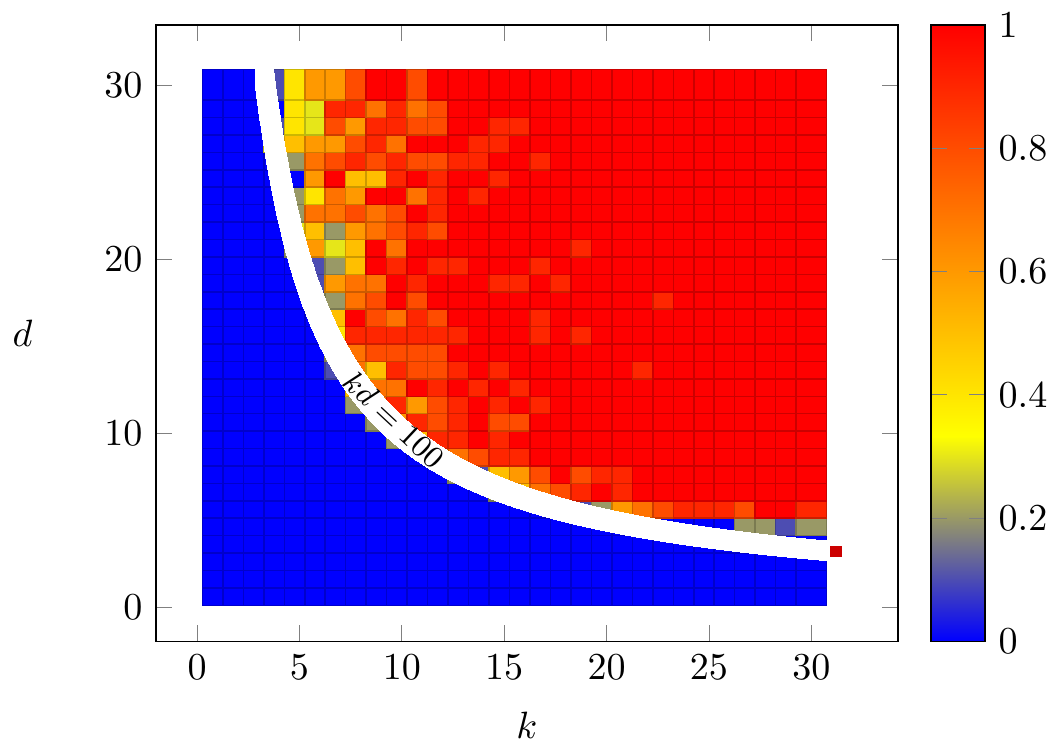}
        \caption{softplus activation with $n=100$}
        \label{PT1S}
    \end{subfigure}
    \begin{subfigure}[b]{0.49\textwidth}
        \includegraphics[width=\textwidth]{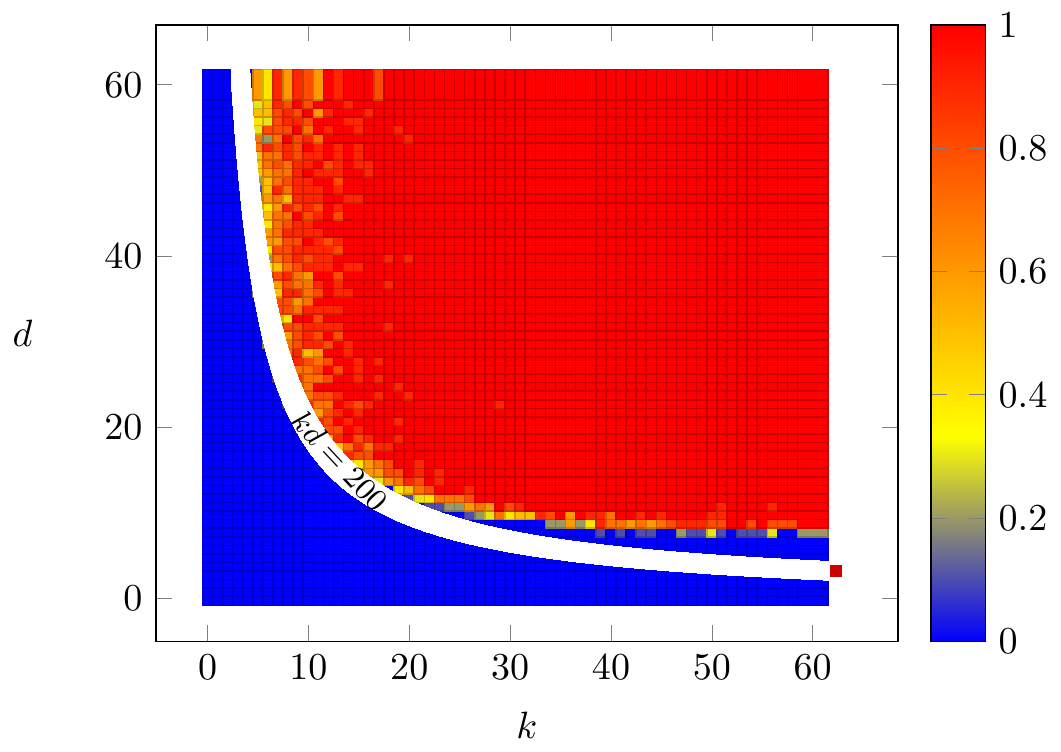}
        \caption{softplus activation with $n=200$}
        \label{PT2S}
    \end{subfigure}
     \begin{subfigure}[b]{0.49\textwidth}
        \includegraphics[width=\textwidth]{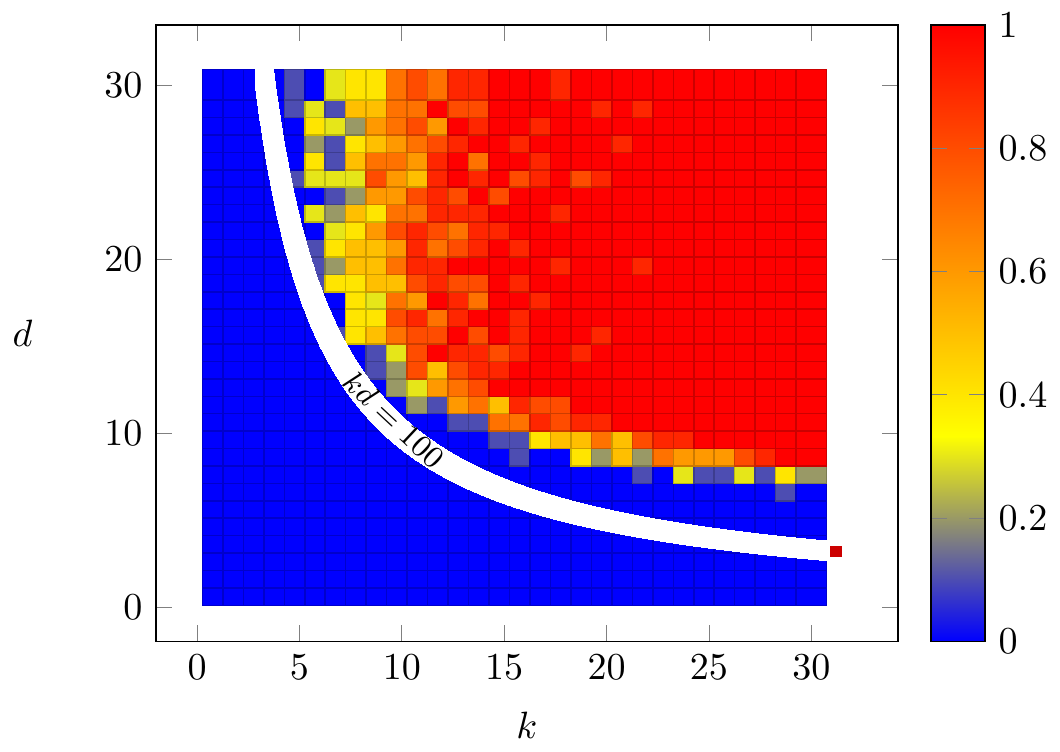}
         \caption{ReLU activation with $n=100$}
        \label{PT1R}
    \end{subfigure}
    \begin{subfigure}[b]{0.49\textwidth}
        \includegraphics[width=\textwidth]{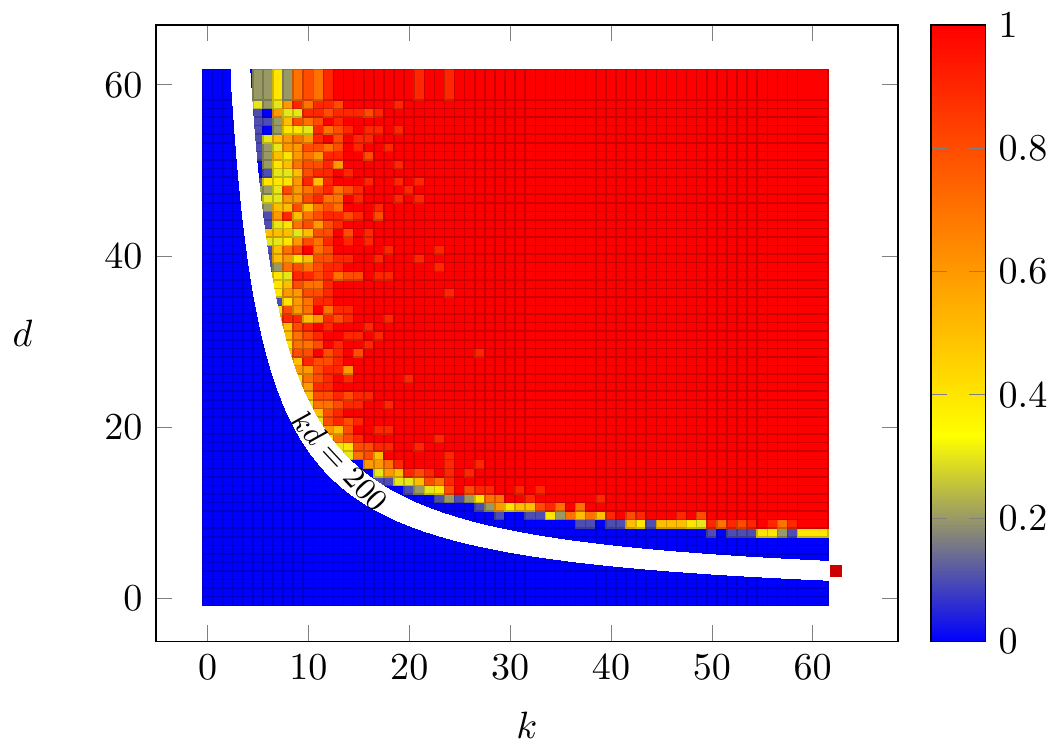}
         \caption{ReLU activation with $n=200$}
        \label{PT2R}
    \end{subfigure}
    \caption{
\textbf{Phase transitions for overparameterization.}
These diagrams show the empirical probability that gradient descent from a random initialization successfully fits $n$ random labels $\y\in\R^n$ when a one-hidden layer neural network is used. Here, $d$ is the input dimension, $k$ the number of hidden units, and $n$ the size of the training data. The colormap tapers between red and blue where red represents certain success, while blue represents certain failure. The solid white line highlights $n=kd$ i.e.~when the size of the training data is equal to the number of parameters.}
    \label{PTcurves}
    \end{figure}
In this section, we provide numerical evidence that neural networks trained with first order methods can fit to random data as long as the number of parameters exceed the size of the training dataset. In particular, we explore the fitting ability of a shallow neural network by fixing a dataset size $n$ and scanning over the different values of hidden nodes $k$ and input dimension $d$. The input samples are drawn i.i.d.~from the unit sphere, the labels i.i.d.~standard normal variables and the input/output weights of the network are initialized according to our theorems. We consider two activations softplus ($\phi(z)=softplus(z)=\log\left(1+e^z\right)$) and Rectified Linear Units ($\phi(z)=ReLU(z)=max(0,z)$). {We pick a constant learning rate of $\eta=0.15$ for softplus and $\eta=0.1$ for ReLU activations. We run the updates for $15000$ iterations or when the relative Euclidean error ($\twonorm{f(\W_\tau)-\y}/\twonorm{\y}$) falls below $2.5\times10^{-3}$. Success is declared if the relative loss is less than $2.5\times 10^{-3}$. To obtain an empirical probability, we average 10 independent realizations for each $(k,d)$ pair.

Figure \ref{PT1S} plots the success probability where $n=100$ and $k$ and $d$ are varied between $0$ to $25$. The solid white line represents the $n=kd$. There is a visible phase transition from failure to success as $k$ and $d$ grows. Perhaps more surprisingly, the success region is tightly surrounded by the $n=kd$ curve indicating that neural nets can overfit as soon as the problem is slightly overparameterized. Figure \ref{PT2S} repeats the same experiment with a larger dataset ($n=200$). Phase transitions are more visible in higher dimensions due to concentration of measure phenomena. Indeed, $n=kd$ curve matches the success region even tighter indicating that $kd>(1+\eps)n$ amount of overparametrization may suffice for fitting random data. 

A related set of experiments are based on assigning random labels in classification problems \cite{zhang2016understanding}. These experiments shuffle the labels of real datasets (e.g.~CIFAR10) and demonstrate that standard deep architectures can still fit them (even if the training takes a bit longer). While these experiments provide very interesting and useful insights the do not address the fundamental tradeoffs surrounding problem parameters such as $n, k, $ and $d$. Finally, we emphasize that the dataset in our experiment is randomly generated. It is possible that worst case datasets exhibit different phase transitions. For instance, if two identical inputs receive different outputs a significantly  higher amounts of overparameterization may be required.
\section{Prior art}

Optimization of neural networks is a challenging problem and it has been the topic of many recent works \cite{zhang2016understanding}. A large body of work focuses on understanding the optimization landscape of the simple nonlinearities or neural networks \cite{zhou2017landscape, soltanolkotabi2017learning, jin2018local,mei2016landscape,brutzkus2017globally,ge2017learning, zhong2017recovery,oymak2018stochastic,fu2018local} when the labels are created according to a planted model. These works establish local convergence guarantees and use techniques such as tensor methods to initialize the network in the proper local neighborhood. Ideally, one would not need specialized initialization if loss surface has no spurious local minima. However, a few publications \cite{yun2018critical,safran2017spurious} demonstrate that the loss surface of nonlinear networks do indeed contains spurious local minima even when the input data are random and the labels are created according to a planted model.


Over-parameterization seems to provide a way to bypass the challenging optimization landscape by relaxing the problem. Several works \cite{zhang2016understanding,Oymak:2018aa,soltanolkotabi2018theoretical,brutzkus2017sgd,keskar2016large,ji2018gradient,song2018mean,sagun2017empirical,chaudhari2016entropy,chizat2018global,arora2018optimization, Ji:2018aa, venturi2018spurious, Zhu:2018aa, Soudry:2016aa} study the benefits of overparameterization for training neural networks and related optimization problems. Very recent works \cite{li2018learning,allen2018convergence,du2018gradient,zou2018stochastic, du2018gradient2} show that overparameterized neural networks can fit the data with random initialization if the number of hidden nodes are polynomially large in the size of the dataset. While these results are based on assuming the networks are sufficiently wide with respect to the size of the data set we only require the total number of parameters to be sufficiently large. Since our conclusions and assumptions are more closely related to \cite{du2018gradient,du2018gradient2} we focus precise comparisons to these two publications. In particular, for smooth activations we show that neural networks can fit the data as soon as $kd\gtrsim n^2$ where as \cite{du2018gradient2} requires $k\gtrsim n^4$. Thus, in terms of the hidden units our results are sharper by a factor on the order of $n^2d$.\footnote{Our results are also sharper in terms of dependence on the quantity $\lambda$ defined in the proofs. In more detail, we require $kd\gtrsim \frac{n^2}{\lambda^2}$ where as \cite{du2018gradient2} requires $k\gtrsim \frac{n^4}{\lambda^4}$.} Focusing on ReLU networks we require $k\gtrsim \frac{n^4}{d^3}$ compared to $k\gtrsim n^6$ assumed in \cite{du2018gradient} so that our results are sharper by a factor $n^2d^3$. Our convergence rate for gradient descent also seems to be faster by a factor on the order of $n$ compared to these results. In addition our results extend to SGD. We would like to note however that our results focus on one-hidden layer networks where as some of the publications above such as \cite{du2018gradient2,allen2018convergence} apply to deep architectures. That said, our results and proof strategy can be extended to deeper architectures and we hope to study such networks in our future work. Finally, these recent papers as well as our work is inherently based on connecting neural networks to kernel methods. We would like to note that the relationship between kernel methods and deep learning has been emphasized by a few interesting publications \cite{belkin2018understand,chizat2018note,Belkin:2018ab,Liang:2018aa}. 

We would also like to note that a few interesting recent papers \cite{mei2018mean, chizat2018global, Sirignano:2018aa, Rotskoff:2018aa} relate the empirical distribution of the network parameters to Wasserstein gradient flows using ideas from mean field analysis. However, this literature is focused on asymptotic characterizations rather than finite-size networks.


An equally important question to understanding the convergence behavior of optimization algorithms for overparameterized models is understanding their generalization capabilities. This is the subject of a few interesting recent papers \cite{neyshabur2017pac,belkin2018reconciling,Arora:2018aa, Bartlett:2017aa, allen2018learning, Golowich:2017aa, Brutzkus:2017aa, Belkin:2018aa, Belkin:2018ab}. While this work do not directly address generalization, techniques developed here (e.g.~characterizing how far is global minima) may help demystify the generalization capabilities of overparametrized networks trained via first order methods. Rigorous understanding of the relationship between optimization and generalization is an interesting and important subject for future research. 

\section{Proofs}
\subsection{Preliminaries}
We begin by noting that for a one-hidden layer neural network of the form $\vct{x}\mapsto \vct{v}^T\phi\left(\W\x\right)$, the Jacobian matrix with respect to vect$(\mtx{W})\in\R^{kd}$ takes the form
\begin{align*}
\mathcal{J}(\mtx{W})=
\begin{bmatrix}
\mathcal{J}(\vct{w}_1) & \ldots & \mathcal{J}(\vct{w}_k)
\end{bmatrix}\in\R^{n\times kd}\quad\text{with}\quad\mathcal{J}(\vct{w}_\ell):=\vct{v}_\ell\text{diag}(\phi'(\mtx{X}\vct{w}_\ell))\mtx{X}.
\end{align*}
Alternatively this can be rewritten in the form
\begin{align}
\label{KR}
\mathcal{J}^T(\mtx{W})=\left(\text{diag}(\vct{v})\phi'\left(\mtx{W}\mtx{X}^T\right)\right)*\mtx{X}^T
\end{align}
An alternative characterization of the Jacobian is 
\begin{align*}
\text{mat}\left(\mathcal{J}^T(\mtx{W})\vct{u}\right)= \text{diag}(\vct{v})\phi'\left(\mtx{W}\mtx{X}^T\right)\text{diag}(\vct{u})\mtx{X}
\end{align*}
In particular, given a residual misfit $\vct{r}:=\vct{r}(\mtx{W}):=\phi\left(\W\X^T\right)^T\vct{v}-\y\in\R^n$ the gradient can be rewritten in the form
\begin{align*}
\nabla \mathcal{L}(\mtx{W})=\text{mat}\left(\mathcal{J}^T(\mtx{W})\vct{r}\right)= \text{diag}(\vct{v})\phi'\left(\mtx{W}\mtx{X}^T\right)\text{diag}(\vct{r})\mtx{X}
\end{align*}
We also note that
\begin{align*}
\mathcal{J}(\mtx{W})\mathcal{J}^T(\mtx{W})=\sum_{\ell=1}^k\vct{v}_\ell^2\text{diag}\left(\phi'\left(\mtx{X}\vct{w}_\ell\right)\right)\mtx{X}\mtx{X}^T\text{diag}\left(\phi'\left(\mtx{X}\vct{w}_\ell\right)\right).
\end{align*}
The latter can also be rewritten in the more compact form
\begin{align*}
\mathcal{J}(\mtx{W})\mathcal{J}^T(\mtx{W})=\left(\phi'\left(\mtx{X}\mtx{W}^T\right)\text{diag}\left(\vct{v}\right)\text{diag}\left(\vct{v}\right)\phi'\left(\mtx{W}\mtx{X}^T\right)\right)\odot\left(\mtx{X}\mtx{X}^T\right).
\end{align*}

\subsection{Meta-theorems}
In this section we will state two meta-theorems and discuss how the two main theorems stated in the main text follow from these results. Our results require defining the notion of a covariance matrix associated to a neural network.
\begin{definition}[Neural network covariance matrix and eigenvalue]\label{nneig} Let $\vct{w}\in\R^d$ be a random vector with a $\mathcal{N}(\vct{0},\mtx{I}_d)$ distribution. Also consider a set of $n$ input data points $\vct{x}_1,\vct{x}_2,\ldots,\vct{x}_n\in\R^d$ aggregated into the rows of a data matrix $\X\in\R^{n\times d}$. Associated to a network $\vct{x}\mapsto\vct{v}^T\phi\left(\W\x\right)$ we and the input data matrix $\X$ we define the neural net covariance matrix as
\begin{align*}
\mtx{\Sigma}(\X):=\E\Big[\left(\phi'\left(\X\w\right)\phi'\left(\X\w\right)^T\right)\odot\left(\X\X^T\right)\Big].
\end{align*}
We also define the eigenvalue $\lambda(\X)$ based on $\mtx{\Sigma}(\X)$ as
\begin{align*}
\lambda(\X):=\lambda_{\min}\left(\mtx{\Sigma}(\X)\right).
\end{align*}
\end{definition}
We note that the neural network covariance matrix is intimately related to the expected value of the Jacobian mapping of the neural network at the random initialization. In particular when the output weights have unit absolute value (i.e.~$\abs{\vct{v}_\ell}=1$), then
\begin{align*}
\mtx{\Sigma}(\X)=\frac{1}{k}\E_{\W_0}\Big[\mathcal{J}(\W_0)\mathcal{J}^T(\W_0)\Big],
\end{align*}
where $\W_0\in\R^{k\times d}$ is a matrix with i.i.d.~$\mathcal{N}(0,1)$ entires. 

As mentioned earlier we prove a more general version of Theorem \ref{thmshallowsmooth} which we now state. The proof is deferred to Section \ref{thmshallowsmoothg}.
\begin{theorem}[Meta-theorem for smooth activations]\label{thmshallowsmoothg}  Consider a data set of input/label pairs $\vct{x}_i\in\R^d$ and $y_i\in\R$ for $i=1,2,\ldots,n$ aggregated as rows/entries of a data matrix $\mtx{X}\in\R^{n\times d}$ and a label vector $\vct{y}\in\R^n$. Without loss of generality we assume the dataset is normalized so that $\twonorm{\vct{x}_i}=1$. Also consider a one-hidden layer neural network with $k$ hidden units and one output of the form $\vct{x}\mapsto \vct{v}^T\phi\left(\mtx{W}\vct{x}\right)$ with $\mtx{W}\in\R^{k\times d}$ and $\vct{v}\in\R^k$ the input-to-hidden and hidden-to-output weights. We assume the activation $\phi$ has bounded derivatives i.e.~$\abs{\phi'(z)}\le B$ and $\abs{\phi''(z)}\le B$ for all $z$. Let $\lambda(\X)$ be the minimum neural net eigenvalue per Definition \ref{nneig}. Furthermore, we set half of the entries of $\vct{v}$ to $\frac{\twonorm{\y}}{\sqrt{kn}}$ and the other half to $-\frac{\twonorm{\y}}{\sqrt{kn}}$ and train only over $\mtx{W}$. Starting from an initial weight matrix $\mtx{W}_0$ selected at random with i.i.d.~$\mathcal{N}(0,1)$ entries we run Gradient Descent (GD) updates of the form $\mtx{W}_{\tau+1}=\mtx{W}_\tau-\eta\nabla \mathcal{L}(\mtx{W}_\tau)$ on the loss \eqref{neuralopt} with step size $\eta= \frac{n\bar{\eta}}{2B^2\tn{\vct{y}}^2\opnorm{\X}^2}$ where $\bar{\eta}\le 1$. Then, as long as
\begin{align}
\label{overparamg}
\sqrt{kd}\ge c B^2(1+\delta) \widetilde{\kappa}(\mtx{X}) n\quad\text{holds with}\quad \widetilde{\kappa}(\mtx{X}):=\frac{\sqrt{\frac{d}{n}}\opnorm{\mtx{X}}}{\lambda\left(\X\right)},
\end{align}
and $c$ a fixed numerical constant, then with probability at least $1-\frac{1}{n}-e^{-\delta^2\frac{n}{2\opnorm{\X}^2}}$ all GD iterates obey
\begin{align*}
&\tn{f(\mtx{W}_\tau)-\y}\leq \left(1-\frac{\bar{\eta}}{32} \frac{1}{B^2}\frac{\lambda\left(\X\right)}{\opnorm{\X}^2}\right)^\tau\tn{f(\mtx{W}_0)-\y},\\
&\frac{\sqrt{\lambda(\X)}}{\sqrt{32}}\frac{\twonorm{\y}}{\sqrt{n}}\fronorm{\mtx{W}_\tau-\mtx{W}_0}+\tn{f(\mtx{W}_\tau)-\y}\leq \tn{f(\mtx{W}_0)-\y}.
\end{align*}
\noindent Furthermore, the total gradient path obeys
\begin{align*}
\sum_{\tau=0}^{\infty} \fronorm{\W_{\tau+1}-\W_\tau}\le \sqrt{32}\frac{\sqrt{n}}{\twonorm{\y}}\frac{\tn{f(\mtx{W}_0)-\y}}{\sqrt{\lambda(\X)}}.
\end{align*}
\end{theorem}
Next we state our meta-theorem for ReLU activations. The proof is deferred to Section \ref{thmshallowsmoothg}}.
\begin{theorem}[Meta-theorem for ReLU activations]\label{reluthmg} Consider the setting of Theorem \ref{thmshallowsmoothg} with the activations equal to $\phi(z)=ReLU(z):=\max(0,z)$ and the $\eta= \frac{n}{3\tn{\y}^2\|\X\|^2}\bar{\eta}$ with $\bar{\eta}\le 1$. Then, as long as
\begin{align}
k\ge C(1+\delta)^2\frac{n\opnorm{\X}^6}{\lambda^4(\X)},\label{k bigg}
\end{align}
holds with $C$ a fixed numerical constant, then with probability at least $1-\frac{2}{n}-e^{-\delta^2\frac{n}{\opnorm{\X}^2}}$ all GD iterates obey
\begin{align*}
&\tn{f(\mtx{W}_\tau)-\y}\leq \left(1-\frac{\bar{\eta}}{24}\frac{\lambda(\X)}{\opnorm{\X}^2}\right)^\tau\tn{f(\mtx{W}_0)-\y},\\
&\frac{1}{6\sqrt{2}}\frac{\twonorm{\y}}{\sqrt{n}}\sqrt{\lambda(\X)}\fronorm{\mtx{W}_\tau-\mtx{W}_0}+\tn{f(\mtx{W}_\tau)-\y}\leq \tn{f(\mtx{W}_0)-\y}.
\end{align*}
\end{theorem}
Our main theorems in Section \ref{sec main resu} can be obtained by substituting the appropriate value of $\la(\X)$ into the two meta theorems above.
\subsection{Reduction to quadratic activations and proofs for Theorems \ref{thmshallowsmooth} and \ref{reluthm}}
Theorems \ref{thmshallowsmooth} and \ref{reluthm} are corollaries of the meta-Theorems \ref{thmshallowsmoothg} and \ref{reluthmg}. To see this connection we will focus on lower bounding the the quantity $\lambda(\X)$ which is not very  interpretable and also not easily computable based on data. In the next lemma we provide a lower bound on $\lambda(\X)$ based on the minimum eigenvalue of the Khatri-Rao product of $\X$ with itself. This key lemma relates the neural network covariance (from Definition \ref{nneig}) for any activation $\phi$ to the case of where the activation is a quadratic of the form $\phi(z)=\frac{1}{2}z^2$. We defer the proof of this lemma to Appendix \ref{reductpf}. We also note that this lemma is a special case of a more general result containing higher order interactions between the data points. Please see Appendix \ref{hermitesec} for more details.
\begin{lemma}[Reduction to quadratic activations]\label{reduct} For an activation $\phi:\R\mapsto \R$ define the quantities
\begin{align*}
\widetilde{\mu}_\phi=\E_{g\sim\mathcal{N}(0,1)}[\phi'(g)]\quad\text{and}\quad\mu_\phi=\E_{g\sim\mathcal{N}(0,1)}[g\phi'(g)].
\end{align*}
Then, the neural network covariance matrix and eigenvalue obey
\begin{align}
\label{keiden}
\mtx{\Sigma}\left(\X\right)\succeq& \left(\widetilde{\mu}_\phi^2\vct{1}\vct{1}^T+\mu_\phi^2\X\X^T\right)\odot\left(\X\X^T\right)\succeq\mu_\phi^2\left(\X\X^T\right)\odot\left(\X\X^T\right),\\
\lambda\left(\X\right)\ge&\mu_\phi^2\sigma_{\min}^2\left(\X*\X\right).
\label{keiden2}
\end{align}
\end{lemma}
To see the relationship with the quadratic activation note that for this activation
\begin{align*}
\mtx{\Sigma}(\X):=&\E\Big[\left(\phi'\left(\X\w\right)\phi'\left(\X\w\right)^T\right)\odot\left(\X\X^T\right)\Big]\\
=&\E\Big[\left(\X\w\w^T\W^T\right)\odot(\X\X^T)\Big]\\
=&\left(\E[\X\w\w^T\W^T]\right)\odot (\X\X^T)\\
=&(\X\X^T)\odot(\X\X^T)\\
=&(\X*\X)(\X*\X)^T.
\end{align*}
Thus the right-hand side of \eqref{keiden} is $\mu_\phi^2$ multiplied by the covariance matrix of a neural network with a quadratic activation $\phi(z)=\frac{1}{2}z^2$.

With this lemma in place we can now prove Theorem \ref{thmshallowsmooth} as simple corollaries of Theorem \ref{thmshallowsmoothg} by noting that $\lambda(\X)\ge \mu_\phi^2\sigma_{\min}^2\left(\X*\X\right)$ per \eqref{keiden2} from Lemma \ref{reduct}. Similarly, to prove Theorem \ref{reluthm} from Theorem \ref{reluthmg} we again use the fact that $\lambda(\X)\ge \mu_\phi^2\sigma_{\min}^2(\X*\X)$ where for the ReLU activation $\mu_\phi^2=\frac{1}{2\pi}$.
\subsection{Lower and upper bounds on the eigenvalues of the Jacobian}
In this section we will state a few key lemmas that provide lower and upper bounds on the eigenvalues of Jacobian matrices. The results in this section apply to any one-hidden neural network with activations that have bounded generalized derivative. In particular, our results here do not require the activation to be differentiable or smooth and thus apply to both the softplus ($\phi(z)=\log\left(e^z+1\right)$) and ReLU ($\phi(z)=\max\left(0,z\right)$) activations.

We begin this section by stating a key lemma regarding the spectrum of the Hadamard product of matrices due to Schur \cite{Schur1911} which plays a crucial role in  both the upper and lower bounds on the eigenvalues of the Jacobian discussed in this section as well as our results on the perturbation of eigenvalues of the Jacobian discussed in the next section.
\begin{lemma}[\cite{Schur1911}]\label{minHad} Let $\mtx{A},\mtx{B}\in\R^{n\times n}$ be two Positive Semi-Definite (PSD) matrices. Then,
\begin{align*}
\lambda_{\min}\left(\mtx{A}\odot\mtx{B}\right)\ge& \left(\min_{i} \mtx{B}_{ii}\right)\lambda_{\min}\left(\mtx{A}\right),\\
\lambda_{\max}\left(\mtx{A}\odot\mtx{B}\right)\le& \left(\max_{i} \mtx{B}_{ii}\right)\lambda_{\max}\left(\mtx{A}\right).
\end{align*}
\end{lemma}
The next lemma focuses on upper bounding the spectral norm of the Jacobian. The proof is deferred to Appendix \ref{spectJpf}.
\begin{lemma}[Spectral norm of the Jacobian]\label{spectJ} Consider a one-hidden layer neural network model of the form $\vct{x}\mapsto \vct{v}^T\phi\left(\W\x\right)$ where the activation $\phi$ has bounded derivatives obeying $\abs{\phi'(z)}\le B$. Also assume we have $n$ data points $\vct{x}_1, \vct{x}_2,\ldots,\vct{x}_n\in\R^d$ aggregated as the rows of a matrix $\X\in\R^{n\times d}$. Then the Jacobian matrix with respect to the input-to-hidden weights obeys 
\begin{align*}
\opnorm{\mathcal{J}(\mtx{W})}\le \sqrt{k}B\infnorm{\vct{v}}\opnorm{\X}.
\end{align*}
\end{lemma}
Next we focus on lower bounding the minimum eigenvalue of the Jacobian matrix at initialization. The proof is deferred to Appendix \ref{minspectJpf}.
\begin{lemma}[Minimum eigenvalue of the Jacobian at initialization]\label{minspectJ} Consider a one-hidden layer neural network model of the form $\vct{x}\mapsto \vct{v}^T\phi\left(\W\x\right)$ where the activation $\phi$ has bounded derivatives obeying $\abs{\phi'(z)}\le B$. Also assume we have $n$ data points $\vct{x}_1, \vct{x}_2,\ldots,\vct{x}_n\in\R^d$ with unit euclidean norm ($\twonorm{\vct{x}_i}=1$) aggregated as the rows of a matrix $\X\in\R^{n\times d}$. Also define $\mu_\phi=\E[g\phi'(g)]$. Then, as long as
\begin{align*}
\frac{\twonorm{\vct{v}}}{\infnorm{\vct{v}}}\ge \sqrt{20\log n}\frac{\opnorm{\X}}{\sqrt{\lambda(\X)}}B,
\end{align*}
the Jacobian matrix at a random point $\mtx{W}_0\in\R^{k\times d}$ with i.i.d.~$\mathcal{N}(0,1)$ entries obeys 
\begin{align*}
\sigma_{\min}\left(\mathcal{J}(\W_0)\right)\ge \frac{1}{\sqrt{2}}\twonorm{\vct{v}}\sqrt{\lambda(\X)},
\end{align*}
with probability at least $1-1/n$.
\end{lemma}
\subsection{Jacobian perturbation}
In this section we discuss results regarding the perturbation of the Jacobian matrix. 

Our first result focuses on smooth activations. In particular, we show the Lipschitz property of the Jacobian with smooth activations. The proof is deferred to Appendix \ref{JLlempf}.
\begin{lemma}[Jacobian Lipschitzness]\label{JLlem} Consider a one-hidden layer neural network model of the form $\vct{x}\mapsto \vct{v}^T\phi\left(\W\x\right)$ where the activation $\phi$ has bounded second order derivatives obeying $\abs{\phi''(z)}\le M$. Also assume we have $n$ data points $\vct{x}_1, \vct{x}_2,\ldots,\vct{x}_n\in\R^d$ with unit euclidean norm ($\twonorm{\vct{x}_i}=1$) aggregated as the rows of a matrix $\X\in\R^{n\times d}$. Then the Jacobian mapping with respect to the input-to-hidden weights obeys
\begin{align*}
\opnorm{\mathcal{J}(\widetilde{\mtx{W}})-\mathcal{J}(\mtx{W})}\le M\infnorm{\vct{v}}\opnorm{\mtx{X}}\fronorm{\widetilde{\mtx{W}}-\mtx{W}}\quad\text{for all}\quad \widetilde{\W},\W\in\R^{k\times d}.
\end{align*}
\end{lemma}

Our second result focuses on perturbation of the Jacobian from the random initialization with ReLU activations. This requires an intricate perturbation bound stated below and proven in Appendix \ref{relpert}.
\begin{lemma} [Jacobian perturbation]\label{thm pert} Consider a one-hidden layer neural network model of the form $\vct{x}\mapsto \vct{v}^T\phi\left(\W\x\right)$ wwith the activation $\phi(z)=ReLU(z):=\max(0,z)$. Also assume we have $n$ data points $\vct{x}_1, \vct{x}_2,\ldots,\vct{x}_n\in\R^d$ with unit euclidean norm ($\twonorm{\vct{x}_i}=1$) aggregated as the rows of a matrix $\X\in\R^{n\times d}$. Also let $\W_0\in\R^{k\times d}$ be a matrix with i.i.d.~$\mathcal{N}(0,1)$ entries and set $m_0=\frac{\tn{\vb}}{\sqrt{200}\tin{\vb}}\frac{\sqrt{\lambda(\X)}}{\opnorm{\X}}$. Then, for all $\W$ obeying 
\[
\opnorm{\W-\W_0}\leq \frac{m_0^{3}}{2k},
\]
with probability at least $1-ne^{-\frac{m_0^2}{6n}}$ the Jacobian matrix $\mathcal{J}$ associated with the neural network obeys
\begin{align}
\|\Jc(\W)-\Jc(\W_0)\|\leq \frac{1}{6\sqrt{2}}\tn{\vb}\sqrt{\lambda(\X)}.
\end{align}
\end{lemma}
\subsection{Proofs for meta-theorem with smooth activations (Proof of Theorem \ref{thmshallowsmoothg})}
To prove this theorem we will utilize a result from \cite{Oymak:2018aa} stated below.
\begin{theorem}\label{GDthm} Consider a nonlinear least-squares optimization problem of the form 
\begin{align*}
\underset{\vct{\theta}\in\R^p}{\min}\text{ }\mathcal{L}(\vct{\theta}):=\frac{1}{2}\twonorm{f(\vct{\theta})-\vct{y}}^2,
\end{align*}
 with $f:\R^p\mapsto \R^n$ and $\vct{y}\in\R^n$. Suppose the Jacobian mapping associated with $f$ obeys
 \begin{align}
 \label{bndspect}
 \bn\le \sigma_{\min}\left(\mathcal{J}(\vct{\theta})\right)\le \|\mathcal{J}(\vct{\theta})\|\le \bp
 \end{align}
over a ball $\mathcal{D}$ of radius $R:=\frac{4\twonorm{f(\vct{\theta}_0)-\vct{y}}}{\bn}$ around a point $\vct{\theta}_0\in\R^p$.\footnote{That is, $\mathcal{D}=\mathcal{B}\left(\vct{\theta}_0,\frac{4\twonorm{f(\vct{\theta}_0)-\vct{y}}}{\bn}\right)$ with $\mathcal{B}(\vct{c},r)=\big\{\vct{\theta}\in\R^p: \twonorm{\vct{\theta}-\vct{c}}\le r\big\}$} Furthermore, suppose 
\begin{align}
\label{lip}
\opnorm{\mathcal{J}(\vct{\theta}_2)-\mathcal{J}(\vct{\theta}_1)}\le L\twonorm{\vct{\theta}_2-\vct{\theta}_1},
\end{align}
holds for any $\vct{\theta}_1,\vct{\theta}_2\in\mathcal{D}$ and set $\eta\leq \frac{1}{2 \bp^2}\cdot\min\left(1,\frac{ \bn^2}{L\twonorm{f(\vct{\theta}_0)-\vct{y}}}\right)$. Then, running gradient descent updates of the form $\vct{\theta}_{\tau+1}=\vct{\theta}_\tau-\eta\nabla\mathcal{L}(\vct{\theta}_\tau)$ starting from $\vct{\theta}_0$, all iterates obey
\begin{align}
\twonorm{f(\vct{\theta}_\tau)-\vct{y}}^2\le&\left(1-\frac{\eta\bn^2}{2}\right)^\tau\twonorm{f(\vct{\theta}_0)-\vct{y}}^2,\label{err}\\
\frac{1}{4}\bn\twonorm{\vct{\theta}_\tau-\vct{\theta}_0}+\twonorm{f(\vct{\theta}_\tau)-\vct{y}}\le&\twonorm{f(\vct{\theta}_0)-\vct{y}}.\label{close}
\end{align}
Furthermore, the total gradient path is bounded. That is,
\begin{align}
\label{GDpath_main}
\sum_{\tau=0}^\infty\twonorm{\vct{\theta}_{\tau+1}-\vct{\theta}_\tau}\le \frac{4\twonorm{f(\vct{\theta}_0)-\vct{y}}}{\bn}.
\end{align}
\end{theorem}
It is more convenient to work with a simpler variation of this theorem that only requires assumption \eqref{bndspect} to hold at the initialization point. We state this corollary below and defer its proof to Appendix \ref{corproof}.
\begin{corollary}\label{maincor} Consider the setting and assumptions of Theorem \ref{GDthm} where 
\begin{align}
\label{initspect}
\sigma_{\min}\left(\mathcal{J}(\vct{\theta}_0)\right)\ge 2\bn,
\end{align}
holds only at the initialization point $\vct{\theta}_0$ in lieu of the left-hand side of \eqref{bndspect}. Furthermore, assume
\begin{align}
\label{corassumption}
\frac{\bn^2}{4L}\ge \twonorm{f(\vct{\theta}_0)-\vct{y}},
\end{align}
holds. Then, the conclusions of Theorem \ref{GDthm} continue to hold.
\end{corollary}
To be able to use this corollary it thus suffices to prove the conditions \eqref{lip}, $\opnorm{\mathcal{J}(\vct{\theta})}\le \bp$, \eqref{initspect}, and \eqref{corassumption} hold for proper choices of $\bn, \bp,$ and $L$. First, by Lemma \ref{JLlem} and our choice of $\vct{v}$ we can use
\begin{align}
\label{Lval}
L=B\infnorm{\vct{v}}\opnorm{\mtx{X}}=\frac{B}{\sqrt{kn}}\twonorm{\y}\opnorm{\X}.
\end{align}  
Second, by Lemma \ref{spectJ} and our choice of $\vct{v}$ we can use
\begin{align}
\label{betaval}
\bp=\sqrt{k}B\infnorm{\vct{v}}\opnorm{\X}=\frac{B}{\sqrt{n}}\twonorm{\vct{y}}\opnorm{\X}.
\end{align}
Next note that
\begin{align}
\label{upplam}
\lambda(\X)=\lambda_{\min}\left(\mtx{\Sigma}(\X)\right)\le \vct{e}_1^T\mtx{\Sigma}(\X)\vct{e}_1=\E_{g\sim\mathcal{N}(0,1)}[\left(\phi'(g)\right)^2]\le B^2\quad\Rightarrow\quad \sqrt{\lambda(\X)}\le B.
\end{align}
Thus, as long as \eqref{overparamg} holds then 
\begin{align*}
\sqrt{k}\ge& c\sqrt{n}B^2\frac{\opnorm{\X}}{\lambda(\X)}\\
\overset{(a)}{\ge}&\sqrt{20 \log n}B^2\frac{\opnorm{\X}}{\lambda(\X)}\\
\overset{(b)}{\ge}&\sqrt{20 \log n}\frac{\opnorm{\X}}{\sqrt{\lambda(\X)}}B.
\end{align*}
Here, (a) follows from the fact that $n\ge \log n$ for $n\ge 1$ and (b) from \eqref{upplam}. Thus by our choice of $\vct{v}$ we have
\begin{align*}
\frac{\twonorm{\vct{v}}}{\infnorm{\vct{v}}}=\sqrt{k}\ge \sqrt{20\log n}B\frac{\opnorm{\X}}{\sqrt{\lambda(\X)}},
\end{align*}
so that Lemma \ref{minspectJ} applies and we can use
\begin{align*}
\bn=\frac{1}{2\sqrt{2}}\twonorm{\vct{v}}\sqrt{\lambda(\X)}=\frac{1}{2\sqrt{2}}\frac{\twonorm{\vct{y}}}{\sqrt{n}}\sqrt{\lambda(\X)}.
\end{align*}
All that remains is to prove the theorem using Corollary \eqref{maincor} is to check that \eqref{corassumption} holds. To this aim we upper bound the initial misfit in the next lemma. The proof is deferred to Section \ref{upreszpf}.
\begin{lemma}[Upper bound on initial misfit]\label{upresz} Consider a one-hidden layer neural network model of the form $\vct{x}\mapsto \vct{v}^T\phi\left(\W\x\right)$ where the activation $\phi$ has bounded derivatives obeying $\abs{\phi'(z)}\le B$. Also assume we have $n$ data points $\vct{x}_1, \vct{x}_2,\ldots,\vct{x}_n\in\R^d$ with unit euclidean norm ($\twonorm{\vct{x}_i}=1$) aggregated as rows of a matrix $\X\in\R^{n\times d}$ and the corresponding labels given by $\vct{y}\in\R^n$. Furthermore, assume we set half of the entries of $\vct{v}\in\R^k$ to $\frac{\twonorm{\y}}{\sqrt{kn}}$ and the other half to $-\frac{\twonorm{\y}}{\sqrt{kn}}$. Then for $\mtx{W}\in\R^{k\times d}$ with i.i.d.~$\mathcal{N}(0,1)$ entries 
\begin{align*}
\twonorm{\phi\left(\X\W^T\right)\vct{v}-\y}\le\twonorm{\y}\left(1+(1+\delta) B\right),
\end{align*}
holds with probability at least $1-e^{-\delta^2\frac{n}{2\opnorm{\X}^2}}$.
\end{lemma}

To do this we will use Lemma \ref{upresz} to conclude that
\begin{align}
\label{ttm1}
\twonorm{f(\vct{\theta}_0)-\vct{y}}:=&\twonorm{\phi\left(\X\W^T\right)\vct{v}-\y}\nn\\
\le&\twonorm{\y}\left(1+(1+\delta) B\right)
\end{align}
holds with probability at least $1-e^{-\delta^2\frac{n}{2\opnorm{\X}^2}}$. Thus, as long as
\begin{align}
\label{ttm2}
\sqrt{kd}\ge& 32 n B\left(1+(1+\delta) B\right)\frac{\sqrt{\frac{d}{n}}\opnorm{\X}}{\lambda(\X)}\nn\\
:=&32B\left(1+(1+\delta) B\right)\widetilde{\kappa}(\X)n
\end{align}
then
\begin{align*}
\frac{\bn^2}{4L}=&\frac{\frac{1}{8}\frac{\twonorm{\y}^2}{n}\lambda(\X)}{4\frac{B}{\sqrt{kn}}\twonorm{\y}\opnorm{\X}}\\
=&\frac{1}{32B}\frac{\sqrt{k}}{\sqrt{n}}\twonorm{\y}\frac{\lambda(\X)}{\opnorm{\X}}\\
=&\frac{1}{32B}\frac{\sqrt{kd}}{\widetilde{\kappa}(\X)n}\twonorm{\y}\\
\ge&\twonorm{\y}\left(1+(1+\delta) B\right).
\end{align*}
Thus, as long as \eqref{overparamg} (equivalent to \eqref{ttm2}) holds, then also \eqref{corassumption} holds and hence $\frac{ \bn^2}{L\twonorm{f(\vct{\theta}_0)-\vct{y}}}\ge 4$. Therefore, using a step size
\begin{align*}
\eta\le \frac{1}{2kB^2\infnorm{\vct{v}}^2\opnorm{\X}^2}=\frac{1}{2\bp^2}=\frac{1}{2\bp^2}\cdot\min(1,4)\le\frac{1}{2 \bp^2}\cdot\min\left(1,\frac{ \bn^2}{L\twonorm{f(\vct{\theta}_0)-\vct{y}}}\right),
\end{align*}
all the assumptions of Corollary \ref{maincor} hold and so do its conclusions, completing the proof of Theorem \ref{thmshallowsmoothg}.

\subsubsection{Upper bounding the initial misfit (Proof of Lemma \ref{upresz})}
\label{upreszpf}
To begin first note that for any two matrices $\widetilde{\W}, \W\in\R^{k\times n}$ we have
\begin{align*}
\abs{\twonorm{\phi\left(\X\widetilde{\W}^T\right)\vct{v}}-\twonorm{\phi\left(\X\W^T\right)\vct{v}}}\le& \twonorm{\phi\left(\X\widetilde{\W}^T\right)\vct{v}-\phi\left(\X\W^T\right)\vct{v}}\\
\le&\opnorm{\phi\left(\X\widetilde{\W}^T\right)-\phi\left(\X\W^T\right)}\twonorm{\vct{v}}\\
\le& \fronorm{\phi\left(\X\widetilde{\W}^T\right)-\phi\left(\X\W^T\right)}\twonorm{\vct{v}}\\
& \hspace{-100pt}\overset{(a)}{=}\fronorm{\left(\phi'\left(\Sb\odot\X\widetilde{\W}^T+(1_{n\times n}-\Sb)\odot\X\W^T\right)\right)\odot \left(\X(\widetilde{\W}-\W)^T\right)}\twonorm{\vct{v}}\\
\le&B\fronorm{\X(\widetilde{\W}-\W)^T}\twonorm{\vct{v}}\\
\le&B\opnorm{\X}\twonorm{\vct{v}}\fronorm{\widetilde{\W}-\W},
\end{align*}
where in (a) we used the mean value theorem with $\Sb$ a matrix with entries obeying $0\le \Sb_{i,j}\le 1$ and $1_{n\times n}$ the matrix of all ones. Thus, $\twonorm{\phi\left(\X\W^T\right)\vct{v}}$ is a $B\opnorm{\X}\twonorm{\vct{v}}$-Lipschitz function of $\mtx{W}$. Thus for a matrix $\W$ with i.i.d.~Gaussian entries 
\begin{align}
\label{tempmyflip}
\twonorm{\phi\left(\X\W^T\right)\vct{v}}\le \E\big[\twonorm{\phi\left(\X\W^T\right)\vct{v}}\big]+t,
\end{align}
holds with probability at least $1-e^{-\frac{t^2}{2B^2\twonorm{\vct{v}}^2\opnorm{\X}^2}}$. We now upper bound the expectation via
\begin{align*}
\E\big[\twonorm{\phi\left(\X\W^T\right)\vct{v}}\big]\overset{(a)}{\le}& \sqrt{\E\big[\twonorm{\phi\left(\X\W^T\right)\vct{v}}^2\big]}\\
=&\sqrt{\sum_{i=1}^n \E\big[\left(\vct{v}^T\phi(\W\vct{x}_i)\right)^2\big]}\\
\overset{(b)}{=}&\sqrt{n}\sqrt{\E_{\vct{g}\sim\mathcal{N}(\vct{0},\mtx{I}_k)}\big[\left(\vct{v}^T\phi(\vct{g})\right)^2\big]}\\
\overset{(c)}{=}&\sqrt{n}\sqrt{\twonorm{\vct{v}}^2\E_{g\sim\mathcal{N}(0,1)}\big[\left(\phi(g)-\E[\phi(g)]\right)^2\big]+(\vct{1}^T\vct{v})^2(\E_{g\sim\mathcal{N}(0,1)}[\phi(g)])^2}\\
\overset{(d)}{=}&\sqrt{n}\twonorm{\vct{v}}\sqrt{\E_{g\sim\mathcal{N}(0,1)}\big[\left(\phi(g)-\E[\phi(g)]\right)^2\big]}\\
\overset{(e)}{\le}&\sqrt{n}B\twonorm{\vct{v}}.
\end{align*}
Here, (a) follows from Jensen's inequality, (b) from linearity of expectation and the fact that for $\vct{x}_i$ with unit Euclidean norm $\mtx{W}\vct{x}_i\sim\mathcal{N}(\vct{0},\mtx{I}_k)$, (c) from simple algebraic manipulations, (d) from the fact that $\vct{1}^T\vct{v}=0$, (e) from $\abs{\phi'(z)}\le B$ a long with the fact that for a $B$-Lipschitz function $\phi$ and normal random variable we have Var$(\phi(g))\le B^2$ based on the Poincare inequality (e.g.~see \cite[p. 49]{ledoux}). Thus using $t=\delta B\sqrt{n}\twonorm{\vct{v}}$ in \eqref{tempmyflip} we conclude that
\begin{align*}
\twonorm{\phi\left(\X\W^T\right)\vct{v}}\le&\twonorm{\vct{v}}\sqrt{n}\left(1+\delta\right)B,\\
=&\twonorm{\y}\left(1+\delta\right)B,
\end{align*}
holds with probability at least $1-e^{-\delta^2\frac{n}{2\opnorm{\X}^2}}$. Thus,
\begin{align*}
\twonorm{\phi\left(\X\W^T\right)\vct{v}-\y}\le \twonorm{\phi\left(\X\W^T\right)\vct{v}}+\twonorm{\y}\le\twonorm{\y}\left(1+(1+\delta) B\right),
\end{align*}
holds with probability at least $1-e^{-\delta^2\frac{n}{2\opnorm{\X}^2}}$ concluding the proof.
\subsection{Proofs for meta-theorem with ReLU activations (Proof of Theorem \ref{reluthmg})}
To prove Theorem \ref{reluthmg} we start by stating a general overparameterized fitting of non-smooth functions. This can be thought of a counter part to Theorem \ref{GDthm} for non-smooth mappings. We note that we do not require the mapping $f$ to be differentiable rather here the Jacobian is defined based on a generalized derivative. Consider a nonlinear least-squares optimization problem of the form 
\begin{align*}
\underset{\vct{\theta}\in\R^p}{\min}\text{ }\mathcal{L}(\vct{\theta}):=\frac{1}{2}\twonorm{f(\vct{\theta})-\vct{y}}^2,
\end{align*}
 with $f:\R^p\mapsto \R^n$ and $\vct{y}\in\R^n$. Suppose the Jacobian mapping associated with $f$ obeys the following three assumptions.
 \begin{assumption} \label{ass1} We assume $\smn{\Jc(\bteta_0)}\geq 2\alpha$ for a point $\bteta_0\in\R^p$.
\end{assumption}
\begin{assumption}  \label{ass3} We assume that for all $\bteta\in\R^d$ we have $\|\Jc(\bteta)\|\leq \beta$.
\end{assumption}
\begin{assumption}  \label{ass2}Let $\|\cdot\|$ denote a norm that is dominated by the Euclidean norm i.e.~$\|\vct{\theta}\|\le \twonorm{\vct{\theta}}$ holds for all $\vct{\theta}\in
\R^p$. Fix a point $\bteta_0$ and a number $R>0$. For any $\bteta$ satisfying $\|\bteta-\bteta_0\|\leq R$, we have that $\|\Jc(\bteta_0)-\Jc(\bteta)\|\leq \alpha/3$.
\end{assumption}

\noindent Under these assumptions we can state the following theorem. We defer the proof of this Theorem to Appendix \ref{nonsmpf}.
\begin{theorem} [Non-smooth Overparameterized Optimization]\label{metathm}Given $\bteta_0\in\R^p$, suppose Assumptions \ref{ass1}, \ref{ass3}, and \ref{ass2} hold with 
\[
R=\frac{3\tn{\y-f(\bteta_0)}}{\alpha}.
\]
Then, using a learning rate $\eta\leq \frac{1}{3\beta^2}$, all gradient iterations obey
\begin{align}
\label{conc1}
&\tn{\y-f(\bteta_{\tau})}\leq \left(1-\eta\alpha^2\right)^{\tau}\tn{\y-f(\bteta_0)},\\
&\frac{\alpha}{3}\|\bteta_{\tau}-\bteta_0\|+\tn{\y-f(\bteta_{\tau})}\leq \tn{\y-f(\bteta_{0})}.
\label{conc2}
\end{align}
\end{theorem}
We shall apply this theorem to the case where the parameter is $\W$, the nonlinear mapping is given by $f(\W)=\vct{v}^T\phi\left(\W\X^T\right)$ with $\phi=ReLU$, and the norm $\|\cdot\|$ is the spectral norm of a matrix. 

\noindent\textbf{Completing the proof of Theorem \ref{reluthmg}.}
With this result in place we are now ready to complete the proof of Theorem \ref{reluthmg}. As in the smooth case \eqref{k bigg} guarantees the condition of Lemma \ref{minspectJ} (i.e.~$k\geq 20\log n\frac{\|\X\|^2}{\lambda(\X)}$) holds. Thus, using Lemma \ref{minspectJ} with probability at least $1-1/n$, Assumption \ref{ass1} holds with 
\begin{align*}
\alpha=\frac{1}{2\sqrt{2n}}\tn{\y}\sqrt{\lambda(\X)}.
\end{align*}
Furthermore, Lemma \ref{spectJ} allows us to conclude that Assumption \ref{ass3} holds with
\begin{align*}
\beta=\frac{1}{\sqrt{n}}\tn{\y}\|\X\|.
\end{align*}
To be able to apply Theorem \ref{metathm}, all that remains is to prove  Assumption \ref{ass2} holds. To this aim note that using Lemma \ref{upresz} with $B=1$ and $\delta\leftarrow2\delta$, to conclude that the initial misfit obeys
\[
\tn{f(\W_0)-\y}\leq 2(1+\delta)\tn{\y},
\]
with probability at least $1-e^{-\delta^2\frac{n}{\|\X\|^2}}$. Therefore, with high probability
\begin{align*}
R:=\frac{3\tn{\y-f(\W_0)}}{\alpha}\le 12(1+\delta)\sqrt{2n}\frac{1}{\sqrt{\lambda(\X)}}.
\end{align*}
%
%
Thus, when \eqref{k bigg} holds using the perturbation Lemma \ref{thm pert} with $m_0=\sqrt{\frac{k}{{200}}}\frac{\sqrt{\lambda(\X)}}{\opnorm{\X}}$, with probability at least $1-ne^{-\frac{k}{1200}\frac{\lambda(\X)}{\opnorm{\X}^2}}-e^{-\delta^2\frac{n}{\|\X\|^2}}$, for all $\W$ obeying
\begin{align*}
\opnorm{\W-\W_0}\leq& R\\
\le& 12(1+\delta)\sqrt{2n}\frac{1}{\sqrt{\lambda(\X)}}\\
\overset{\eqref{k bigg}}{\le}&\frac{\sqrt{k}\lambda^{\frac{3}{2}}(\X)}{2(200)^{\frac{3}{2}}\opnorm{\X}^3}\\
=&\frac{m_0^3}{2k}
\end{align*}
we have
\[
\|\Jc(\W)-\Jc(\W_0)\|\leq \frac{1}{6\sqrt{2n}}\tn{\y}\sqrt{\lambda(\X)}=\frac{\alpha}{3}.
\]
This guarantees Assumption \ref{ass2} also holds concluding the proof of Theorem \ref{reluthmg} via Theorem \ref{metathm}. 
%


\subsection{Proofs for training the output layer (Proof of Theorem \ref{thm sense})}\label{pfoutput}
To begin note that
\begin{align*}
\mtx{\Phi}\mtx{\Phi}^T=\phi\left(\mtx{X}\mtx{W}^T\right)\phi\left(\W\mtx{X}^T\right)=\sum_{\ell=1}^k \phi\left(\mtx{X}\vct{w}_\ell\right)\phi\left(\mtx{X}\vct{w}_\ell\right)^T\succeq \sum_{\ell=1}^k \phi\left(\mtx{X}\vct{w}_\ell\right)\phi\left(\mtx{X}\vct{w}_\ell\right)^T\mathbb{1}_{\{\twonorm{\phi(\X\w_\ell)}\le T_n\}}.
\end{align*}
Here $T_n$ a function of $n$ whose value shall be determined later in the proofs. To continue we need the matrix Chernoff result stated below.
\begin{theorem}[Matrix Chernoff] Consider a finite sequence $\mtx{A}_\ell\in\R^{n\times n}$ of independent, random, Hermitian matrices with common dimension $n$. Assume that $\mtx{0}\preceq \mtx{A}_\ell\preceq R\mtx{I}$ for $\ell=1,2,\ldots,k$. Then
\begin{align*}
\mathbb{P}\Bigg\{\lambda_{\min}\left(\sum_{\ell=1}^k \mtx{A}_\ell\right)\le (1-\delta)\lambda_{\min}\left(\sum_{\ell=1}^k\E[\mtx{A}_\ell]\right)\Bigg\}\le n\left(\frac{e^{-\delta}}{(1-\delta)^{(1-\delta)}}\right)^{\frac{\lambda_{\min}\left(\sum_{\ell=1}^k\E[\mtx{A}_\ell]\right)}{R}}
\end{align*}
for $\delta\in[0,1)$.
\end{theorem}
Applying this theorem with $\mtx{A}_\ell=\phi\left(\mtx{X}\vct{w}_\ell\right)\phi\left(\mtx{X}\vct{w}_\ell\right)^T\mathbb{1}_{\{\twonorm{\phi(\X\w_\ell)}\le T_n\}}$, $R=T_n^2$ and $\widetilde{\mtx{A}}(\w):=\phi\left(\mtx{X}\vct{w}\right)\phi\left(\mtx{X}\vct{w}\right)^T\mathbb{1}_{\big\{\twonorm{\phi(\X\w)}\le T_n\big\}}$
\begin{align}
\label{maintrunin}
\lambda_{\min}\left(\mtx{\Phi}\mtx{\Phi}^T\right)\ge (1-\delta)k\lambda_{\min}\left(\E[\widetilde{\mtx{A}}(\w)]\right),
\end{align}
holds with probability at least $1-n\left(\frac{e^{-\delta}}{(1-\delta)^{(1-\delta)}}\right)^{\frac{k\lambda_{\min}\left(\E[\widetilde{\mtx{A}}(\w)]\right)}{T_n^2}}$.

Next we shall connect the the expected value of the truncated matrix $\widetilde{\mtx{A}}(\w)$ to one that is not truncated defined as $\mtx{A}(\w)=\phi(\X\w)\phi(\X\w)^T$. To do this note that
\begin{align}
\label{intertrunc}
\opnorm{\E[\widetilde{\mtx{A}}(\w)-\mtx{A}(\w)]}=&\opnorm{\E\Big[\phi(\X\w)\phi(\X\w)^T\mathbb{1}_{\big\{\twonorm{\phi(\X\w)}> T_n\big\}}\Big]}\nn\\
\overset{(a)}{\le}&\E\Big[\opnorm{\phi(\X\w)\phi(\X\w)^T\mathbb{1}_{\big\{\twonorm{\phi(\X\w)}> T_n\big\}}}\Big]\nn\\
\le&\E\Big[\twonorm{\phi(\X\w)}^2\mathbb{1}_{\big\{\twonorm{\phi(\X\w)}> T_n\big\}}\Big]\\
\overset{(b)}{\le}&2\E\Big[\twonorm{\phi(\X\w)-\phi(\vct{0})}^2\mathbb{1}_{\big\{\twonorm{\phi(\X\w)}> T_n\big\}}\Big]+2\E\Big[\twonorm{\phi(\vct{0})}^2\mathbb{1}_{\big\{\twonorm{\phi(\X\w)}> T_n\big\}}\Big]\nn\\
\overset{(c)}{\le}&2B^2\E\Big[\twonorm{\X\w}^2\mathbb{1}_{\big\{\twonorm{\phi(\X\w)}> T_n\big\}}\Big]+2nB^2\mathbb{P}\{\twonorm{\phi(\X\w)}> T_n\}\nn\\
\overset{(d)}{\le}&2B^2\sqrt{\E\big[\twonorm{\X\w}^4\big]\mathbb{P}\{\twonorm{\phi(\X\w)}> T_n\}}+2nB^2\mathbb{P}\{\twonorm{\phi(\X\w)}> T_n\}\nn\\
\overset{(e)}{\le}&2\sqrt{n}B^2\sqrt{\left(\sum_{i=1}^n \E\big[\abs{\x_i^T\w}^4\big]\right)\mathbb{P}\{\twonorm{\phi(\X\w)}> T_n\}}+2nB^2\mathbb{P}\{\twonorm{\phi(\X\w)}> T_n\}\nn\\
\overset{(f)}{\le}&2\sqrt{3}nB^2\sqrt{\mathbb{P}\{\twonorm{\phi(\X\w)}> T_n\}}+2nB^2\mathbb{P}\{\twonorm{\phi(\X\w)}> T_n\}\nn\\
\le&6nB^2\sqrt{\mathbb{P}\{\twonorm{\phi(\X\w)}> T_n\}}.
\end{align}
Here, (a) follows from Jensen's inequality, (b) from the simple identity $(a+b)^2\le 2(a^2+b^2)$, (c) from $\abs{\phi'(z)}\le B$, (d) from the Cauchy-Schwarz inequality, (e) from Jensen's inequality, and (f) from the fact that for a standard moment random variable $X$ we have $\E[X^4]=3$. 

To continue we need to show that $\mathbb{P}\{\twonorm{\phi(\X\w)}> T_n\}$ is small. To this aim note that for any activation $\phi$ with $\abs{\phi'(z)}\le B$ we have
\begin{align*}
\twonorm{\phi\left(\mtx{X}\vct{w}_2\right)-\phi\left(\mtx{X}\vct{w}_1\right)}\le B\opnorm{\mtx{X}}\twonorm{\vct{w}_2-\vct{w}_1}
\end{align*}
Thus by Lipschitz concentration of Gaussian functions for a random vector $\vct{w}\sim\mathcal{N}(\vct{0},\mtx{I}_d)$ we have
\begin{align*}
\twonorm{\phi\left(\mtx{X}\vct{w}\right)}\le& \E[\twonorm{\phi\left(\mtx{X}\vct{w}\right)}]+t,\\
\le&\sqrt{\E[\twonorm{\phi\left(\mtx{X}\vct{w}\right)}^2]}+t,\\
=&\sqrt{n}\sqrt{\E_{g\sim\mathcal{N}(0,1)}[\phi^2(g)]}+t,\\
\le& B\sqrt{2n}+t,
\end{align*}
holds with probability at least $1-e^{-\frac{t^2}{2B^2\opnorm{\mtx{X}}^2}}$. Thus using $t=\Delta B\sqrt{n}$ we conclude that
\begin{align*}
\twonorm{\phi\left(\mtx{X}\vct{w}\right)}\le (\Delta+\sqrt{2})B\sqrt{n},
\end{align*}
holds with probability at least $1-e^{-\frac{\Delta^2}{2}\frac{n}{\opnorm{\mtx{X}}^2}}$. Thus using $\Delta=c\sqrt{\log n}$ and $T_n=CB\sqrt{n\log n}$ we can conclude that
\begin{align*}
\mathbb{P}\{\twonorm{\phi(\X\w)}> T_n\}\le \frac{1}{n^{202}}.
\end{align*}
Thus, using \eqref{intertrunc} we can conclude that
\begin{align*}
\opnorm{\E[\widetilde{\mtx{A}}(\w)-\mtx{A}(\w)]}\le \frac{6B}{n^{100}}.
\end{align*}
Combining this with \eqref{maintrunin} with $\delta=1/2$ we conclude that
\begin{align*}
\lambda_{\min}\left(\mtx{\Phi}\mtx{\Phi}^T\right)\ge \frac{1}{2}k\left(\lambda_{\min}\left(\E[\mtx{A}(\w)]\right)-\frac{6B}{n^{100}}\right)=\frac{1}{2}k\left(\widetilde{\lambda}(\X)-\frac{6B}{n^{100}}\right),
\end{align*}
holds with probability at least $1-ne^{-\gamma \frac{k\widetilde{\lambda}(\X)}{T_n^2}}$. The latter probability is larger than $1-\frac{1}{n^{100}}$ as long as 
\begin{align*}
k\ge C \log^2(n)\frac{n}{\widetilde{\lambda}(\X)},
\end{align*}
concluding the proof.
\section*{Acknowledgements}
M. Soltanolkotabi would like to thank the Modest Yachts \#mathshop slack channel for fruitful discussions. In particular, Laurant Lessard, Ali Rahimi, and Ben Recht who pointed out via plots that applying softplus to a Gaussian input leads to essentially a uniform distribution. M. Soltanolkotabi would like to thank Zixuan Zhang for help with the simulations of Figure \ref{PTcurves}. M. Soltanolkotabi is supported by the Packard Fellowship in Science and Engineering, an NSF-CAREER under award \#1846369, the Air Force Office of Scientific Research Young Investigator Program (AFOSR-YIP)
under award \#FA9550-18-1-0078, an NSF-CIF award \#1813877, and a Google faculty research award. 
\bibliography{Bibfiles}
\bibliographystyle{unsrt} 

\appendix
\section{Proofs for bounding the eigenvalues of the Jacobian}
\subsection{Proof for the spectral norm of the Jacobian (Proof of Lemma \ref{spectJ})}
\label{spectJpf}
To bound the spectral norm note that as stated earlier 
\begin{align*}
\mathcal{J}(\mtx{W})\mathcal{J}^T(\mtx{W})=\left(\phi'\left(\mtx{X}\mtx{W}^T\right)\text{diag}\left(\vct{v}\right)\text{diag}\left(\vct{v}\right)\phi'\left(\mtx{W}\mtx{X}^T\right)\right)\odot\left(\mtx{X}\mtx{X}^T\right).
\end{align*}
Thus using Lemma \ref{minHad} we have
\begin{align*}
\opnorm{\mathcal{J}(\mtx{W})}^2\le& \left(\max_i\text{ }\twonorm{\text{diag}\left(\vct{v}\right)\phi'\left(\mtx{W}\vct{x}_i\right)}^2\right)\lambda_{\max}\left(\X\X^T\right)\\
=& \left(\max_i\text{ }\twonorm{\text{diag}\left(\vct{v}\right)\phi'\left(\mtx{W}\vct{x}_i\right)}^2\right)\opnorm{\X}^2\\
\le&\infnorm{\vct{v}}^2\left(\max_i\text{ }\twonorm{\phi'\left(\mtx{W}\vct{x}_i\right)}^2\right)\opnorm{\X}^2\\
\le&kB^2\infnorm{\vct{v}}^2\opnorm{\X}^2,
\end{align*}
completing the proof.
\subsection{Proofs for minimum eigenvalue of the Jacobian at initialization (Proof of Lemma \ref{minspectJ})}
\label{minspectJpf}
To lower bound the minimum eigenvalue of $\mathcal{J}(\W_0)$, we focus on lower bounding the minimum eigenvalue of $\mathcal{J}(\W_0)\mathcal{J}(\W_0)^T$. To do this we first lower bound the minimum eigenvalue of the expected value $\E\big[\mathcal{J}(\W_0)\mathcal{J}(\W_0)^T\big]$ and then related the matrix $\mathcal{J}(\W_0)\mathcal{J}(\W_0)^T$ to its expected value. We proceed by simplifying the expected value. To this aim we use the identity
\begin{align*}
\mathcal{J}(\mtx{W})\mathcal{J}^T(\mtx{W})=&\left(\phi'\left(\mtx{X}\mtx{W}^T\right)\text{diag}\left(\vct{v}\right)\text{diag}\left(\vct{v}\right)\phi'\left(\mtx{W}\mtx{X}^T\right)\right)\odot\left(\mtx{X}\mtx{X}^T\right)\\
=&\left(\sum_{\ell=1}^k\vct{v}_\ell^2\phi'\left(\X\vct{w}_\ell\right)\phi'\left(\X\vct{w}_\ell\right)^T\right)\odot \left(\X\X^T\right),
\end{align*}
mentioned earlier to conclude that
\begin{align}
\E\big[\mathcal{J}(\W_0)\mathcal{J}(\W_0)^T\big]=&\twonorm{\vct{v}}^2\left(\E_{\vct{w}\sim\mathcal{N}(0,\mtx{I}_d)}\big[\phi'\left(\X\vct{w}\right)\phi'\left(\X\vct{w}\right)^T\big]\right)\odot\left(\X\X^T\right),\nn\\
:=&\twonorm{\vct{v}}^2\Sigma\left(\X\right).
\end{align}
Thus
\begin{align}
\label{mineigexp}
\lambda_{\min}\left(\E\big[\mathcal{J}(\W_0)\mathcal{J}(\W_0)^T\big]\right)\ge\twonorm{\vct{v}}^2\lambda(\X).
\end{align}
To relate the minimum eigenvalue of the expectation to that of $\mathcal{J}(\W_0)\mathcal{J}(\W_0)^T$ we utilize the matrix Chernoff identity stated below.
\begin{theorem}[Matrix Chernoff] Consider a finite sequence $\mtx{A}_\ell\in\R^{n\times n}$ of independent, random, Hermitian matrices with common dimension $n$. Assume that $\mtx{0}\preceq \mtx{A}_\ell\preceq R\mtx{I}$ for $\ell=1,2,\ldots,k$. Then
\begin{align*}
\mathbb{P}\Bigg\{\lambda_{\min}\left(\sum_{\ell=1}^k \mtx{A}_\ell\right)\le (1-\delta)\lambda_{\min}\left(\sum_{\ell=1}^k\E[\mtx{A}_\ell]\right)\Bigg\}\le n\left(\frac{e^{-\delta}}{(1-\delta)^{(1-\delta)}}\right)^{\frac{\lambda_{\min}\left(\sum_{\ell=1}^k\E[\mtx{A}_\ell]\right)}{R}}
\end{align*}
for $\delta\in[0,1)$.
\end{theorem}
We shall apply this theorem with $\mtx{A}_\ell:=\mathcal{J}(\vct{w}_\ell)\mathcal{J}^T(\vct{w}_\ell)=\vct{v}_\ell^2\text{diag}(\phi'(\mtx{X}\vct{w}_\ell))\mtx{X}\mtx{X}^T\text{diag}(\phi'(\mtx{X}\vct{w}_\ell))$. To this aim note that
\begin{align*}
\vct{v}_\ell^2\text{diag}(\phi'(\mtx{X}\vct{w}_\ell))\mtx{X}\mtx{X}^T\text{diag}(\phi'(\mtx{X}\vct{w}_\ell))\preceq B^2\infnorm{\vct{v}}^2\opnorm{\mtx{X}}^2\mtx{I},
\end{align*}
so that we can use Chernoff Matrix with $R=B^2\infnorm{\vct{v}}^2\opnorm{\mtx{X}}^2$ to conclude that
\begin{align*}
\mathbb{P}\Bigg\{\lambda_{\min}\left(\mathcal{J}(\W_0)\mathcal{J}^T(\W_0)\right)\le (1-\delta)\lambda_{\min}\left(\E\big[\mathcal{J}(\W_0)\mathcal{J}^T(\W_0)\big]\right)\Bigg\}\\
\quad\quad\quad\quad\le n\left(\frac{e^{-\delta}}{(1-\delta)^{(1-\delta)}}\right)^{\frac{\lambda_{\min}\left(\E\big[\mathcal{J}(\W_0)\mathcal{J}^T(\W_0)\big]\right)}{B^2\infnorm{\vct{v}}^2\opnorm{\mtx{X}}^2}}.
\end{align*}
Thus using \eqref{mineigexp} in the above with $\delta=\frac{1}{2}$ we have
\begin{align*}
\mathbb{P}\Bigg\{\lambda_{\min}\left(\mathcal{J}(\W_0)\mathcal{J}^T(\W_0)\right)\le \frac{1}{2}\twonorm{\vct{v}}^2\lambda\left(\X\right)\Bigg\}\le n\cdot e^{-\frac{1}{10}\frac{\twonorm{\vct{v}}^2\lambda(\X)}{B^2\infnorm{\vct{v}}^2\opnorm{\mtx{X}}^2}}.
\end{align*}
Therefore, as long as
\begin{align*}
\frac{\twonorm{\vct{v}}}{\infnorm{\vct{v}}}\ge \sqrt{20\log n}\frac{\opnorm{\X}}{\sqrt{\lambda(\X)}}B,
\end{align*}
then
\begin{align*}
\sigma_{\min}\left(\mathcal{J}(\W_0)\right)\ge \frac{1}{\sqrt{2}}\twonorm{\vct{v}}\sqrt{\lambda(\X)},
\end{align*}
holds with probability at leat $1-\frac{1}{n}$.

\section{Reduction to quadratic activations (Proof of Lemma \ref{reduct})}
\label{reductpf}
First we note that \eqref{keiden2} simply follows from \eqref{keiden} by noting that 
\[
(\X\X^T)\odot(\X\X^T)=\left(\X*\X\right)\left(\X*\X\right)^T.
\] Thus we focus on proving \eqref{keiden}. We begin the proof by noting two simple identities. First, using multivariate Stein identity we have
\begin{align}
\label{temp371}
\E\Big[\mtx{X}\vct{w}\phi'\left(\mtx{X}\vct{w}\right)^T\Big]=&\sum_{i=1}^n\E\Big[\mtx{X}\vct{w}\cdot\phi'\left(\vct{e}_i^T\mtx{X}\vct{w}\right)\Big]\vct{e}_i^T\nn\\
=&\sum_{i=1}^n\mtx{X}\mtx{X}^T\E\Big[\phi''\left(\vct{e}_i^T\mtx{X}\vct{w}\right)\vct{e}_i\Big]\vct{e}_i^T\nn\\
=&\mtx{X}\mtx{X}^T\text{diag}\left(\E[\phi''(\mtx{X}\vct{w})]\right)\nn\\
=&\E_{g\sim\mathcal{N}(0,1)}[\phi''(g)]\mtx{X}\mtx{X}^T\nn\\
=&\E_{g\sim\mathcal{N}(0,1)}[g\phi'(g)]\mtx{X}\mtx{X}^T\nn\\
=&\mu_\phi\mtx{X}\mtx{X}^T,
\end{align}
where in the last line we used the fact that $\twonorm{\vct{x}_i}=1$. We note that while for clarity of exposition we carried out the above proof using the fact that $\phi'$ is differentiable the identity above continues to hold without assuming $\phi'$ is differentiable with a simple modification to the above proof. Next we note that
\begin{align}
\label{temp372}
\E[\phi'(\mtx{X}\vct{w})]=\E_{g\sim\mathcal{N}(0,1)}[\phi'(g)]\vct{1}:=\widetilde{\mu}_\phi.
\end{align}
We continue by noting that
\begin{align}
\label{psdiden}
\E\Big[\left(\phi'\left(\mtx{X}\vct{w}\right)-\eta\vct{1}-\gamma \mtx{X}\vct{w} \right)\left(\phi'\left(\mtx{X}\vct{w}\right)-\eta\vct{1}-\gamma \mtx{X}\vct{w}\right)^T\Big]\succeq \mtx{0}.
\end{align}
Thus, using \eqref{temp371} and \eqref{temp372} we have
\begin{align*}
\quad\quad\quad\quad\quad\quad\E&\Big[\left(\phi'\left(\mtx{X}\vct{w}\right)-\eta\vct{1}-\gamma \mtx{X}\vct{w} \right)\left(\phi'\left(\mtx{X}\vct{w}\right)-\eta\vct{1}-\gamma \mtx{X}\vct{w}\right)^T\Big]\\
=&\E\Big[\phi'\left(\mtx{X}\vct{w}\right)\phi'\left(\mtx{X}\vct{w}\right)^T\Big]-2\eta\widetilde{\mu}_\phi\vct{1}\vct{1}^T-2\gamma \mu_\phi\mtx{X}\mtx{X}^T\\
&+\eta^2\vct{1}\vct{1}^T+\gamma^2\mtx{X}\mtx{X}^T\\
=&\E\Big[\phi'\left(\mtx{X}\vct{w}\right)\phi'\left(\mtx{X}\vct{w}\right)^T\Big]+\eta\left(\eta-2\widetilde{\mu}_\phi\right)\vct{1}\vct{1}^T\\
&+\gamma\left(\gamma-2\mu_\phi\right)\mtx{X}\mtx{X}^T.
\end{align*}
Combining the latter with \eqref{psdiden} we arrive at 
\begin{align*}
\E\Big[\phi'\left(\mtx{X}\vct{w}\right)\phi'\left(\mtx{X}\vct{w}\right)^T\Big]\succeq \eta \left(2\widetilde{\mu}_\phi-\eta\right)\vct{1}\vct{1}^T+\gamma\left(2\mu_\phi-\gamma\right)\mtx{X}\mtx{X}^T.
\end{align*}
Hence, setting $\eta=\widetilde{\mu}_\phi$ and $\gamma=\mu_\phi$ we conclude that 
\begin{align*}
\E\Big[\phi'\left(\mtx{X}\vct{w}\right)\phi'\left(\mtx{X}\vct{w}\right)^T\Big]\succeq \widetilde{\mu}_\phi^2\vct{1}\vct{1}^T+\mu_\phi^2\mtx{X}\mtx{X}^T.
\end{align*}
Thus
\begin{align*}
\mtx{\Sigma}\left(\X\right)=&\left(\E\Big[\phi'\left(\mtx{X}\vct{w}\right)\phi'\left(\mtx{X}\vct{w}\right)^T\Big]\right)\odot(\X\X^T)\\
\succeq& \left(\widetilde{\mu}_\phi^2\vct{1}\vct{1}^T+\mu_\phi^2\mtx{X}\mtx{X}^T\right)\odot(\X\X^T)\\
\succeq& \mu_\phi^2(\X\X^T)\odot(\X\X^T)
\end{align*}
completing the proof of \eqref{keiden} and the lemma.

\section{Proofs for Jacobian perturbation}
\subsection{Proof for Lipschitzness of the Jacobian with smooth activations (Proof of Lemma \ref{JLlem})}
\label{JLlempf}
To prove this lemma first note that using the form \eqref{KR} we have
\begin{align*}
\mathcal{J}\left(\widetilde{\mtx{W}}\right)-\mathcal{J}\left(\mtx{W}\right)=\left(\text{diag}(\vct{v})\left(\phi'\left(\mtx{X}\widetilde{\mtx{W}}^T\right)-\phi'\left(\mtx{X}\mtx{W}^T\right)\right)\right)*\mtx{X}.
\end{align*}
Now using the fact that $(\mtx{A}*\mtx{B})(\mtx{A}*\mtx{B})^T=\left(\mtx{A}\mtx{A}^T\right)\odot \left(\mtx{B}\mtx{B}^T\right)$ we conclude that
\begin{align}
\label{JJT}
&\left(\mathcal{J}\left(\widetilde{\mtx{W}}\right)-\mathcal{J}\left(\mtx{W}\right)\right)\left(\mathcal{J}\left(\widetilde{\mtx{W}}\right)-\mathcal{J}\left(\mtx{W}\right)\right)^T\nonumber\\
&\quad\quad\quad\quad=\left(\left(\phi'\left(\mtx{X}\widetilde{\mtx{W}}^T\right)-\phi'\left(\mtx{X}\mtx{W}^T\right)\right)\text{diag}(\vct{v})\text{diag}(\vct{v})\left(\phi'\left(\widetilde{\mtx{W}}\mtx{X}^T\right)-\phi'\left(\mtx{W}\mtx{X}^T\right)\right)\right)\nn\\
&\hspace{60pt}\odot\left(\mtx{X}\mtx{X}^T\right).
\end{align}
To continue further we use Lemma \ref{minHad} combined with \eqref{JJT} to conclude that 
\begin{align*}
\opnorm{\mathcal{J}\left(\widetilde{\mtx{W}}\right)-\mathcal{J}\left(\mtx{W}\right)}^2\le&\opnorm{\text{diag}(\vct{v})\left(\phi'\left(\widetilde{\mtx{W}}\mtx{X}^T\right)-\phi'\left(\mtx{W}\mtx{X}^T\right)\right)}^2\left(\max_i \twonorm{\vct{x}_i}^2\right)\\
\le&\infnorm{\vct{v}}^2\opnorm{\phi'\left(\widetilde{\mtx{W}}\mtx{X}^T\right)-\phi'\left(\mtx{W}\mtx{X}^T\right)}^2\\
\overset{(a)}{=}&\infnorm{\vct{v}}^2\opnorm{\phi''\left((\mtx{S}\odot\mtx{W}+(1-\mtx{S})\odot\widetilde{W})\mtx{X}^T\right)\odot\left((\widetilde{\W}-\W)\X^T\right)}^2\\
\le&\infnorm{\vct{v}}^2\fronorm{\phi''\left((\mtx{S}\odot\mtx{W}+(1-\mtx{S})\odot\widetilde{W})\mtx{X}^T\right)\odot\left((\widetilde{\W}-\W)\X^T\right)}^2\\
\le&\infnorm{\vct{v}}^2B^2\fronorm{(\widetilde{\W}-\W)\X^T}^2\\
\le&\infnorm{\vct{v}}^2B^2\opnorm{\X}^2\fronorm{\widetilde{\W}-\W}^2,
\end{align*}
completing the proof of this lemma. Here, (a) holds by the mean value theorem for some matrix $\mtx{S}\in\R^{k\times d}$ with entries $0\le \mtx{S}_{ij}\le 1$.
\subsection{Jacobian perturbation results for ReLU networks (Proof of Lemma \ref{thm pert})}
\label{relpert}
To prove Lemma \ref{thm pert} we first relate the perturbation of the Jacobian to perturbation of the activation pattern $\phi'(\X\W^T)$ as follows.
\begin{lemma} \label{j to phip}Consider the matrices $\W,\widetilde{\W}\in\R^{k\times d}$ and a data matrix $\X\in\R^{n\times d}$ with unit Euclidean norm rows. Then,
\[
\|\Jc(\W)-\Jc(\widetilde{\W})\|\leq \tin{\vb}\opnorm{\X}\cdot\underset{1\le i\le n}{\max}\twonorm{\phi'\left(\W\x_i\right)-\phi'\left(\widetilde{\W}\x_i\right)}.
\]
\end{lemma}
\begin{proof} Similar to the smooth case in the previous section, the Jacobian difference is given by
\[
\Jc(\W)-\Jc(\widetilde{\W})=\left(\text{diag}(\vb)\left(\phi'(\X\W^T)-\phi'(\X\widetilde{\W}^T)\right)\right)*\X.
\]
Consequently, 
\begin{align*}
\opnorm{\Jc(\W)-\Jc(\widetilde{\W})}^2&=\opnorm{\left(\Jc(\W)-\Jc(\widetilde{\W})\right)\left(\Jc(\W)-\Jc(\widetilde{\W})\right)^T}\\
&\hspace{-33pt}\le \opnorm{(\left(\phi'\left(\X\W^T\right)-\phi'\left(\X\widetilde{\W}^T\right)\right)\text{diag}(\vct{v})\text{diag}(\vct{v})\left(\phi'\left(\W\X^T\right)-\phi'\left(\widetilde{W}\X^T\right)\right)\right)\\
&\hspace{-25pt}~~~~\odot\left(\X\X^T)}\\
&\le\left(\underset{1\le i\le n}{\max} \twonorm{\text{diag}(\vct{v})\left(\phi'\left(\W\x_i\right)-\phi'\left(\widetilde{W}\x_i\right)\right)}^2\right)\cdot \opnorm{\X}^2\\
&\le \tin{\vb}^2\opnorm{\X}^2\cdot\underset{1\le i\le n}{\max}\twonorm{\phi'\left(\W\x_i\right)-\phi'\left(\widetilde{\W}\x_i\right)}^2
\end{align*}
\end{proof}
The lemma above implies that, we simply need to control $\phi'(\W\x_i)$ around a neighborhood of $\W_0$. To continue note that since $\phi'$ is the step function, we shall focus on the number of sign flips between the matrices $\W\X^T$ and $\W_0\X^T$. Let $\|\vb\|_{m-}$ denote the $m$th smallest entry of $\vb$ after sorting its entries in terms of absolute value. We first state a intermediate lemma.
\begin{lemma} \label{simple m lem}Given an integer $m$, suppose
\[
\opnorm{\W-\W_0}\leq\sqrt{{m}} \abs{\W_0\x_i}_{m-},
\]
holds for $i=1,2,\ldots,n$. Then 
\[
\underset{1\le i\le n}{\max}\twonorm{\phi'(\W\x_i)-\phi'(\W_0\x_i)}\le \sqrt{2m}.
\]
\end{lemma}
\begin{proof} We will prove this result by contradiction. Suppose there is an $\vct{x}_i$ such that $\phi'(\W\x_i)$ and $\phi'(\W_0\x_i)$ have (at least) $2m$ different entries. Let $\{(a_r, b_r)\}_{r=1}^{2m}$ be (a subset of) entries of $\W\vct{x}_i,\W_0\x_i$ at these differing locations respectively and suppose $a_r$'s are sorted decreasingly in absolute value. By definition $\abs{a_{r}}\geq \abs{\W_0\x_i}_{m-}$ for $r\leq m$. Consequently, using $\text{sign}(a_r)\neq \text{sign}(b_r)$,
\begin{align*}
\opnorm{\W-\W_0}^2&\geq \twonorm{(\W-\W_0)\vct{x}_i}^2\\
&\geq \sum_{r=1}^{2m} |a_r-b_r|^2\\
&\geq \sum_{r=1}^{2m} |a_r|^2\\
&\geq m\abs{\W_0\x_i}_{m-}^2.
\end{align*}
This implies $\opnorm{\W-\W_0}\geq\sqrt{{m}} \abs{\W_0\x_i}_{m-}$ contradicting the assumption of the lemma and thus concluding the proof.
\end{proof}
Now note that by setting $m=m_0^2$ in Lemma \ref{simple m lem} as long as
\begin{align}
\label{m0 sign eq}
\opnorm{\W-\W_0}\leq m_0 \abs{\W_0\x_i}_{m_0^2-},
\end{align} 
we have
\begin{align}
\underset{1\le i\le n}{\max}\twonorm{\phi'(\W\x_i)-\phi'(\W_0\x_i)}\leq \frac{\tn{\vb}}{10\tin{\vb}}\frac{\sqrt{\lambda(\X)}}{\opnorm{\X}}:=\sqrt{2}m_0.\label{req req}
\end{align}
Using Lemma \ref{j to phip}, this in turn implies
\begin{align*}
\opnorm{\Jc(\W)-\Jc(\W_0)}&\leq \tin{\vb}\opnorm{\X}\cdot\underset{1\le i\le n}{\max}\twonorm{\phi'\left(\W\x_i\right)-\phi'\left(\W_0\x_i\right)}
\leq  \frac{1}{10}\tn{\vb}\sqrt{\lambda(\X)}\\
&\leq \frac{1}{6\sqrt{2}}\tn{\vb}\sqrt{\lambda(\X)}.
\end{align*}
Thus to complete the proof of Lemma \ref{thm pert} all that remains is to prove \eqref{m0 sign eq}. To this aim, we state the following lemma proven later in this section.
\begin{lemma}\label{entry control} Let $\vct{x}_1,\vct{x}_2,\ldots,\vct{x}_n\in\R^{d}$ be the input data point with unit Euclidean norm. Also let $\W_0\in\R^{k\times d}$ be a matrix with i.i.d.~$\mathcal{N}(0,1)$ entries. Then, with probability at least $1-ne^{-\frac{m}{6}}$,
\[
\abs{\W_0\vct{x}_i}_{m-}\geq \frac{m}{2k}\quad \text{for all}\quad i=1,2,\ldots,n.
\]
\end{lemma}
Now applying Lemma \ref{entry control} we conclude that with probability at least $1-ne^{-\frac{1}{1200}\frac{\twonorm{\vct{v}}^2}{\infnorm{\vct{v}}^2}\frac{\lambda(\X)}{\opnorm{\X}^2}}$
\[
{m_0}\abs{\W_0\x_i}_{m_0^2-}\geq \frac{m_0^{3}}{2k},
\]
holds for all $i=1,2,\ldots,n$. Hence, with same probability, all $\opnorm{\W-\W_0}\leq \frac{m_0^{3}}{2k}$ obeys \eqref{m0 sign eq} concluding the proof of Lemma \ref{thm pert}.
\subsubsection{Proof of Lemma \ref{entry control}}
Observe that $\W_0\vct{x}_1,\W_0\vct{x}_2,\ldots,\W_0\vct{x}_n$ are all standard normal however they depend on each other. We begin by focusing on one such vector. We begin by proving that with probability at least $1-e^{-\frac{m}{6}}$, at most $m$ of the entries of $\W_0\x_i$ are less than $\frac{m}{2k}$. To this aim let $\gamma_{\alpha}$ be the number for which $\mathbb{P}\{|g|\leq \gamma_{\alpha}\}=\alpha$ where $g\sim\Nn(0,1)$ (i.e.~the inverse cumulative density function of $|g|$). $\gamma_{\alpha}$ trivially obeys $\gamma_{\alpha}\geq\sqrt{\pi/2} \alpha$. To continue set $\vct{g}:=\W_0\x_i\sim\Nn(0,\Iden_k)$ and the Bernouli random variables $\delta_\ell$ given by
\begin{align*}
\delta_\ell=\left\{
	\begin{array}{ll}
		1  & \mbox{if } \abs{g_\ell}\le \gamma_\delta \\
		0 & \mbox{if } \abs{g_\ell} > \gamma_\delta
	\end{array}
\right.
\end{align*}
with $\delta=\frac{m}{2k}$. Note that
\begin{align*}
\E\Bigg[\sum_{\ell=1}^k \delta_\ell\Bigg]=\sum_{\ell=1}^k\E[\delta_\ell]=\sum_{\ell=1}^k\mathbb{P}\big\{\abs{g_\ell}\le\gamma_\delta\big\}=\delta k=\frac{m}{2}.
\end{align*}
Since the $\delta_\ell$'s are i.i.d., applying a standard Chernoff bound we obtain
\beq
\mathbb{P}\Bigg\{\sum_{\ell=1}^k \delta_\ell\ge m\Bigg\}\le e^{-\frac{m}{6}}.\nn
\eeq
The complementary event implies that at most $m$ entries are less than $\frac{m}{2k}$. This together with the union bound completes the proof.
\section{Proof of Corollary \ref{maincor}}
\label{corproof}
First note that \eqref{corassumption} can be rewritten in the form
\begin{align*}
R:=\frac{4}{\bn}\twonorm{f(\vct{\theta}_0)-\vct{y}}\le \frac{\bn}{L}.
\end{align*}
Thus using the Lipschitzness of the Jacobian from \eqref{lip} for all $\vct{\theta}\in\mathcal{B}\left(\vct{\theta}_0,R\right)$ we have
\begin{align*}
\opnorm{\mathcal{J}(\vct{\theta})-\mathcal{J}(\vct{\theta}_0)}\le L\twonorm{\vct{\theta}-\vct{\theta}_0}\le LR\le \bn.
\end{align*}
Combining the latter with the triangular inequality we conclude that
\begin{align*}
\smn{\mathcal{J}(\vct{\theta})}\ge \smn{\mathcal{J}(\vct{\theta}_0)}-\opnorm{\mathcal{J}(\vct{\theta})-\mathcal{J}(\vct{\theta}_0)} \ge 2\bn-\bn=\bn,
\end{align*}
so that \eqref{bndspect} holds under the assumptions of the corollary. Therefore, all of the assumptions of Theorem \ref{GDthm} continue to hold and thus so do its conclusions.

\section{Proof of Corollary \ref{maincor2}}
\label{corpf}
The proof follows from a simple application of Theorem \ref{thmshallowsmooth}. We just need to calculate the various constants involved in this result. First, we focus on the constants related to the activation.It is trivial to check that $B=M=1$ and $\mu_\phi\approx0.207$ so that \eqref{overparam} reduces to $\sqrt{kd}\ge \tilde{c}(1+\delta)\kappa(\X)n$ with $\tilde{c}$ a fixed numerical constant. Next we focus on the constant $\kappa(\mtx{X})$ that depends on the data matrix. To this aim we note that standard results regarding the concentration of spectral norm of random matrices with i.i.d.~rows imply that
\begin{align*}
\opnorm{\X}\le 2\sqrt{\frac{n}{d}},
\end{align*}
holds with probability at least $1-e^{-\gamma_2 d}$. Furthermore, based on a simple modification of \cite[Corollary 6.5]{soltanolkotabi2018theoretical} 
\begin{align}
\sigma_{\min}\left(\X*\X\right)\ge c,
\end{align}
holds with probability at least $1-ne^{-\gamma_1\sqrt{n}}-\frac{1}{n}-2ne^{-\gamma_2 d}$ where $c, \gamma_1,$ and $\gamma_2$ are fixed numerical constants.
\section{Proof of Theorem \ref{SGDthmnn}}
The proof of this result follows from \cite[Theorem 3.1]{Oymak:2018aa} similar to how Theorem \ref{thmshallowsmooth} follows from Theorem \ref{GDthm} from the same paper. The only new parameter we have to calculate is the maximum Euclidean norm of the rows of the Jacobian matrix. For neural networks this takes the form
\begin{align*}
\underset{i}{\max}\text{ }\twonorm{\mathcal{J}_i(\mtx{W})}=&\fronorm{\text{diag}(\vct{v})\phi'\left(\W\vct{x}_i\right)\vct{x}_i^T}\\
=&\fronorm{\text{diag}(\vct{v})\phi'\left(\W\vct{x}_i\right)}\twonorm{\x_i}\\
=&\fronorm{\text{diag}(\vct{v})\phi'\left(\W\vct{x}_i\right)}\\
\le&\infnorm{\phi'\left(\W\vct{x}_i\right)}\twonorm{\vct{v}}\\
\le&B\twonorm{\vct{v}}.
\end{align*}
\section{Proofs for nonsmooth optimization (Proof of Theorem \ref{metathm})}
\label{nonsmpf}
To prove this theorem we begin by stating a few preliminary results and definitions.
\begin{lemma} [Asymmetric PSD perturbation]\label{asym pert} Consider the matrices $\A,\B,\Cb\in\R^{n\times p}$ obeying $\|\B-\Cb\|\le \eps$ and $\|\A-\Cb\|\leq \eps$. Then, for all $\rb\in\R^n$,
\[
|\rb^T\B\A^T\rb- \tn{\Cb^T\rb}^2|\leq2\eps\tn{\Cb^T\rb}\tn{\rb}+\eps^2\tn{\rb}^2.
\]
\end{lemma}
\begin{proof} We have
\begin{align*}
\rb^T\B\A^T\rb-\tn{\Cb^T\rb}^2&=\rb^T(\B-\Cb)(\A-\Cb)^T\rb+\rb^T\Cb (\A-\Cb)^T\rb+\rb^T(\B-\Cb)\Cb^T\rb.
\end{align*}
This implies
\begin{align*}
|\rb^T\B\A^T\rb-\tn{\Cb^T\rb}^2|&\le|\rb^T(\B-\Cb)(\A-\Cb)^T\rb|+\tn{(\A-\Cb)^T\rb}\tn{\Cb^T\rb}\\
&~~~~+\tn{(\B-\Cb)^T\rb}\tn{\Cb^T\rb}\\
&\leq \eps^2\tn{\rb}^2+2\eps\tn{\Cb^T\rb}\tn{\rb},
\end{align*}
concluding the proof.
\end{proof}
\begin{definition} [Average Jacobian] \label{avg jacob}We define the average Jacobian along the path connecting two points $\vct{x},\vct{y}\in\R^p$ as
\begin{align}
&\Jc(\y,\x):=\int_0^1 \mathcal{J}(\x+\alpha(\y-\x))d\alpha.
\end{align}
\end{definition}
\begin{lemma} \label{lem control}Suppose $\x,\y\in\R^p$ satisfy $\|\x-\bteta_0\|,\|\y-\bteta_0\|\leq R$. Then, under Assumptions \ref{ass1} and \ref{ass2}, for any $\rb\in\R^d$, we have
\begin{align*}
&\rb^T\Jc(\y,\x)\Jc(\x)^T\rb\geq \frac{\tn{\Jc(\bteta_0)^T\rb}^2}{2},\\
&\tn{\Jc(\x)^T\rb}^2\leq 1.5\tn{\Jc(\bteta_0)^T\rb}^2.
\end{align*}
\end{lemma}
\begin{proof} Under Assumptions \ref{ass1} and \ref{ass2}, applying Lemma \ref{asym pert} with $\A=\Jc(\x)$, $\B=\Jc(\y,\x)$, $\Cb=\Jc(\bteta_0)$, and $\eps=\alpha/3$, we conclude that
\begin{align*}
\rb^T\Jc(\y,\x)\Jc(\x)^T\rb-\tn{\Jc(\bteta_0)^T\rb}^2&\geq -\left(\frac{2\alpha}{3} \tn{\Jc(\bteta_0)^T\rb}\tn{\rb}+\frac{\alpha^2}{9}\tn{\rb}^2\right)\\
&\geq -\left(\frac{2\alpha}{3} \tn{\Jc(\bteta_0)^T\rb}\tn{\rb}+\frac{\alpha}{18}\tn{\Jc(\bteta_0)^T\rb}\tn{\rb}\right)\\
&\geq -\alpha  \tn{\Jc(\bteta_0)^T\rb}\tn{\rb}\\
&\geq - \frac{\tn{\Jc(\bteta_0)^T\rb}^2}{2}.
\end{align*}
This implies $\rb^T\Jc(\y,\x)\Jc(\x)^T\rb\geq \frac{\tn{\Jc(\bteta_0)^T\rb}^2}{2}$. The upper bound similarly follows from Lemma \ref{asym pert} by setting $\A=\B=\Jc(\x)$ and observing that the deviation is again upper bounded by $\frac{\tn{\Jc(\bteta_0)^T\rb}^2}{2}$.
\end{proof}
\begin{lemma}\label{lem 1} Suppose Assumptions \ref{ass1} and \ref{ass2} hold. Consider two consequent iterative updates $\vct{\theta}_\tau$ and $\vct{\theta}_{\tau+1}$ which by definition obey
\begin{align*}
\vct{\theta}_{\tau+1}:=\vct{\theta}_\tau-\eta \mathcal{J}^T(\vct{\theta}_\tau)\left(f(\vct{\theta}_\tau)-\y\right),
\end{align*}
with $\eta\le \frac{1}{3\beta^2}$. Also, denote the corresponding residuals by $\vct{r}_{\tau+1}:=f(\vct{\theta}_{\tau+1})-\y$ and $\vct{r}_{\tau}:=f(\vct{\theta}_{\tau})-\y$. Finally, assume $\bteta_{\tau},\bteta_{\tau+1}$ satisfy $\|\bteta_{\tau+1}-\bteta_0\|,\|\bteta_{\tau}-\bteta_0\|\leq R$. Then
\[
\tn{\rb_{\tau+1}}\leq \tn{\rb_{\tau}}-\frac{\eta}{4}   \frac{ \tn{\Jc(\bteta_0)^T\rb}^2}{\tn{\rb}}.
\]
\end{lemma}
\begin{proof} For this proof we use the short-hand $\Jc_{\tau+1,\tau}:=\mathcal{J}(\vct{\theta}_\tau,\vct{\theta}_\tau)$ and $\mathcal{J}_\tau:=\mathcal{J}(\vct{\theta}_\tau)$. We expand the residual at $\bteta_{\tau+1}$ using Lemma \ref{lem control} as follows
\begin{align*}
\tn{\rb_{\tau+1}}^2&=\tn{(\Iden-\eta\Jc_{\tau+1,\tau}\Jc_{\tau}^T)\rb_{\tau}}^2\\
&=\tn{\rb_{\tau}}^2-2\eta \rb_{\tau}^T\Jc_{\tau+1,\tau}\Jc_{\tau}^T\rb_{\tau}+\eta^2\tn{\Jc_{\tau+1,\tau}\Jc_{\tau}^T\rb_{\tau}}^2\\
&\leq \tn{\rb_{\tau}}^2-\eta  \tn{\Jc(\bteta_0)^T\rb}^2+\eta^2\bp^2\tn{\Jc_{\tau}^T\rb_{\tau}}^2\\
&\leq \tn{\rb_{\tau}}^2-\eta  \tn{\Jc(\bteta_0)^T\rb}^2+\frac{3}{2}\eta^2\bp^2\tn{\Jc(\bteta_0)^T\rb_{\tau}}^2
\end{align*}
Using the fact that $\eta\leq \frac{1}{3\bp^2}$, we conclude that
\[
\tn{\rb_{\tau+1}}^2\leq \tn{\rb_{\tau}}^2-\frac{\eta}{2}  \tn{\Jc(\bteta_0)^T\rb}^2\implies \tn{\rb_{\tau+1}}\leq \tn{\rb_{\tau}}-\frac{\eta}{4} \frac{ \tn{\Jc(\bteta_0)^T\rb}^2}{\tn{\rb}}.
\]
\end{proof}

\subsection{Completing the proof of Theorem \ref{metathm}}
With these lemmas in place we are now ready to complete the proof of Theorem \ref{metathm}. To this aim suppose the conclusions hold until iteration $\tau>0$. We shall show the result for iteration $\tau+1$. We first prove that iterates still stays inside the region $\|\bteta-\bteta_0\|\leq R$. To this aim first note that by the induction hypothesis we know that
\[
\|\bteta_{\tau}-\bteta_0\|\leq R-\frac{3\tn{\y-f(\bteta_{\tau})}}{\alpha}.
\]
Combining this with the gradient update rule, $\eta\leq 1/\beta^2$ and $\|\Jc\|\leq\beta$ yields
\begin{align*}
\|\bteta_{\tau+1}-\bteta_0\|&\leq \|\bteta_{\tau}-\bteta_0\|+\eta \|\Jc(\bteta_\tau)\rb_\tau\|\\
&\leq \|\bteta_{\tau}-\bteta_0\|+\eta \twonorm{\Jc(\bteta_\tau)\rb_\tau}\\
&\leq R-\frac{3\tn{\y-f(\bteta_{\tau})}}{\alpha}+\eta \tn{\Jc(\bteta_\tau)\rb_\tau}\\
&\leq R-\frac{3\tn{\y-f(\bteta_{\tau})}}{\alpha}+\frac{1}{\beta} \tn{\rb_\tau}\\
&\le R.
\end{align*}
Now that we have shown $\|\bteta_{\tau+1}-\bteta_0\|\leq R$, we can apply Lemma \ref{lem 1} to conclude that
\begin{align}
\tn{\rb_{\tau+1}}\leq \tn{\rb_{\tau}}-\frac{\eta}{4}   \frac{ \tn{\Jc(\bteta_0)^T\rb}^2}{\tn{\rb}}\leq  \tn{\rb_{\tau}}-\frac{\eta\alpha}{2} \tn{\Jc(\bteta_0)^T\rb}.\label{rhss}
\end{align}
Next, we complement this by using Lemma \ref{lem control} to control the increase in the distance of the iterates to the initial point. This allows us to conclude that
\begin{align*}
\|\bteta_{\tau+1}-\bteta_0\|&\leq \|\bteta_{\tau}-\bteta_0\|+\eta \|\grad{\bteta_\tau}\|,\\
\|\bteta_{\tau+1}-\bteta_0\|&\leq \|\bteta_{\tau}-\bteta_0\|+\eta \tn{\grad{\bteta_\tau}},\\
&\leq \|\bteta_{\tau}-\bteta_0\|+\eta \tn{\Jc^T(\bteta_{\tau})\rb_{\tau}},\\
&\leq \|\bteta_{\tau}-\bteta_0\|+1.25\eta \tn{\Jc^T(\bteta_{0})\rb_{\tau}}.
\end{align*}
Adding the latter two identities, we obtain
\[
\tn{\rb_{\tau+1}}+\frac{\alpha}{3}\|\bteta_{\tau+1}-\bteta_0\|\leq \tn{\rb_{\tau}}+\frac{\alpha}{3}\|\bteta_{\tau}-\bteta_0\|\leq \tn{\rb_0},
\]
completing the proof of \eqref{conc2}. Finally, the convergence rate guarantee \eqref{conc1} follows from \eqref{rhss} can be upper bounded by $(1-\eta\alpha^2)\tn{\rb_{\tau}}$.
\section{Lower bounds on the minimum eigenvalue of covariance matrices}
\label{hermitesec}
In this section we discuss lower bounds on the minimum eigenvalue of the neural network and output feature covariance matrices which involve higher order Khatri-Rao products. This results involve the Hermite expansion of the activation and its derivatives. For any $\phi$ with bounded Gaussian meaure i.e.~$\frac{1}{\sqrt{2\pi}}\int_{-\infty}^{+\infty} \phi^2(g)e^{-\frac{g^2}{2}}dg<\infty$ the Hermite coefficients $\{\mu_r(\phi)\}_{r=0}^{+\infty}$ associated to $\phi$ are defined as 
\begin{align*}
\mu_r(\phi):=\frac{1}{\sqrt{2\pi}}\int_{-\infty}^{+\infty} \phi(g)h_r(g)e^{-\frac{g^2}{2}}dg,
\end{align*}
where $h_r(g)$ is the normalized probabilists' Hermite polynomial defined by 
\begin{align*}
h_r(x):=\frac{1}{\sqrt{r!}}(-1)^r e^{\frac{x^2}{2}}\frac{d^r}{dx^r}e^{-\frac{x^2}{2}}.
\end{align*}
Using these expansions we prove the following simple lemma. The first one is a generalization of the reduction to quadratic activation Lemma (Lemma \ref{reduct}). We note that Lemma \ref{reduct} is a special case as $\widetilde{\mu}_\phi=\mu_0(\phi)$ and $\mu_\phi=\mu_1(\phi)$.
\begin{lemma}\label{reduct2} For an activation $\phi:\R\mapsto \R$ and a data matrix $\X\in\R^{n\times d}$ with unit Euclidean norm rows the neural network covariance matrix and eigenvalue obey
\begin{align}
\label{keidenh}
\mtx{\Sigma}\left(\X\right)=& \left(\mu_0^2(\phi')\vct{1}\vct{1}^T+\sum_{r=1}^{+\infty}\mu_r^2(\phi)\left(\X\X^T\right)^{\odot r}\right)\odot\left(\X\X^T\right)\succeq\mu_r^2(\phi')\left(\X\X^T\right)^{\odot(r+1)},\\
\lambda\left(\X\right)\ge&\mu_r^2(\phi')\sigma_{\min}^2\left(\X^{*(r+1)}\right)\quad\text{ for any }r=0,1,2,\ldots.
\label{keidenh2}
\end{align}
As a reminder, for a matrix $\mtx{A}\in\R^{n\times n}$, $\mtx{A}^{\odot r}\in\R^{n\times n}$ is defined inductively via $\mtx{A}^{\odot r}=\mtx{A}\odot \left(\mtx{A}^{\odot (r-1)}\right)$ with $\mtx{A}^{\odot 0}=\vct{1}\vct{1}^T$. Similarly, for a matrix $\X\in\R^{n\times d}$ with rows given by $\vct{x}_i\in\R^d$ we define the matrix $\X^{*r}\in\R^{n\times d^r}$ as
\begin{align*}
\big[\X^{*r}\big]_i=\left(\underbrace{\vct{x}_i\otimes \vct{x}_i\otimes \ldots \otimes \vct{x}_i}_{r} \right)^T 
\end{align*}
\end{lemma}
\begin{proof} To prove this result note that by the properties of Hermite expansions we have
\begin{align*}
\Big[\E[\phi'\left(\mtx{X}\vct{w}\right)\phi'\left(\mtx{X}\vct{w}\right)^T]\Big]_{ij}=&\E[\phi'(\vct{x}_i^T\vct{w})\phi'(\vct{x}_j^T\vct{w})]\\
=&\sum_{r=0}^\infty \mu_r^2(\phi')(\vct{x}_i^T\vct{x}_j)^r
\end{align*}
Thus
\begin{align*}
\mtx{\Sigma}\left(\X\right)=&\left(\sum_{r=0}^\infty \mu_r^2(\phi')\left(\mtx{X}\mtx{X}^T\right)^{\odot r}\right)\odot\left(\X\X^T\right).
\end{align*}
Furthermore,
\begin{align*}
\sum_{r=0}^\infty \mu_r^2(\phi')\left(\mtx{X}\mtx{X}^T\right)^{\odot r}=\sum_{r=0}^\infty \left(\mu_r(\phi')\mtx{X}^{*r}\right)\left(\mu_r(\phi')\mtx{X}^{*r}\right)^T\succeq \mu_r^2(\phi')\left(\mtx{X}^{*r}\right)\left(\mtx{X}^{*r}\right)^T=\mu_r^2(\phi')\left(\X\X^T\right)^{\odot r}.
\end{align*}
Using the latter combined with the fact that the Hadamard product of two PSD matrices are PSD we arrive at \eqref{keidenh}. The latter also implies \eqref{keidenh2}.
\end{proof}
Similarly, it is also easy to prove the following result about the output feature covariance.
\begin{lemma}\label{reduct2} For an activation $\phi:\R\mapsto \R$ and a data matrix $\X\in\R^{n\times d}$ with unit Euclidean norm rows the output feature covariance matrix and eigenvalue obey
\begin{align}
\label{sigphi}
\widetilde{\mtx{\Sigma}}\left(\X\right)=& \left(\mu_0^2(\phi)\vct{1}\vct{1}^T+\sum_{r=1}^{+\infty}\mu_r^2(\phi)\left(\X\X^T\right)^{\odot r}\right)\succeq\mu_r^2(\phi)\left(\X\X^T\right)^{\odot(r)},\\
\widetilde{\lambda}\left(\X\right)\ge&\mu_r^2(\phi)\sigma_{\min}^2\left(\X^{*r}\right)\quad\text{ for any }r=1,2,\ldots.
\label{sigphi2}
\end{align}
\end{lemma}
\begin{proof} To prove this result note that by the properties of Hermite expansions we have
\begin{align*}
\Big[\E[\phi\left(\mtx{X}\vct{w}\right)\phi\left(\mtx{X}\vct{w}\right)^T]\Big]_{ij}=&\E[\phi(\vct{x}_i^T\vct{w})\phi(\vct{x}_j^T\vct{w})]\\
=&\sum_{r=0}^\infty \mu_r^2(\phi)(\vct{x}_i^T\vct{x}_j)^r
\end{align*}
Thus
\begin{align*}
\widetilde{\mtx{\Sigma}}\left(\X\right)=&\sum_{r=0}^\infty \mu_r^2(\phi)\left(\mtx{X}\mtx{X}^T\right)^{\odot r}\succeq \mu_r^2(\phi)\left(\X\X^T\right)^{\odot(r)},
\end{align*}
concluding the proof of \eqref{sigphi}. This in turn also implies \eqref{sigphi2}.
\end{proof}
\section{Proofs for datasets with $\delta$-separation (Proof of Theorem \ref{reluthmsep})}\label{sec sep}
We begin by stating a result regarding the covariance of the indicator mapping. Below we use $\mathcal{I}$ to denote the indicator mapping i.e.~$\mathcal{I}(z)=\mathbb{1}_{\{z\ge 0\}}$.
\begin{theorem} \label{sep thm main}Let $\x_1,\dots,\x_n$ be points in $\R^d$ with unit Euclidian norm and $\w\sim\Nn(0,\Iden_d)$. Form the matrix $\X\in\R^{n\times d}=[\x_1~\dots~\x_n]^T$. Suppose there exists $\delta>0$ such that for every $1\leq i\neq j\leq n$ we have that
\[
\min(\tn{\x_i-\x_j},\tn{\x_i+\x_j})\geq \delta.
\]
Then, the covariance of the vector $\ind{\X\w}$ obeys
\begin{align}\label{require}
\E[\ind{\X\w}\ind{\X\w}^T]\succeq \frac{\delta}{100n^2}.
\end{align}
\end{theorem}
\begin{proof} Fix a unit length vector $\ab\in\R^n$. Suppose there exists constants $c_1,c_2$ such that
\begin{align}\label{needed}
\Pro(|\ab^T\ind{\X\w}|\geq c_1\tin{\ab})\geq \frac{c_2\delta }{n}.
\end{align}
This would imply that 
\[
\E[(\ab^T\ind{\X\w})^2]\geq \E[|\ab^T\ind{\X\w}|]^2\geq c_1^2\tin{\ab}^2\frac{c_2\delta}{n}\geq c_1^2c_2\frac{\delta}{n^2}.
\]
Since this is true for all $\ab$, we find \eqref{require} with $c_1^2c_2=\frac{1}{100}$ by choosing $c_1=1/2,c_2=1/25$ as described later. Hence, our goal is proving \eqref{needed}. For the most part, our argument is based on exploiting independence of orthogonal decomposition associated with Gaussian vectors and we will refine the argument of \cite{zou2018stochastic}. Without losing generality, assume $|a_1|=\tin{\ab}$ and construct an orthonormal basis $\Qb$ in $\R^d$ where the first column is equal to $\x_1$ and $\Qb=[\x_1~\bar{\Qb}]$. Note that $\g=\Qb^T\w\sim\Nn(0,\Iden_d)$ and we have
\[
\w=\Qb\g=g_1\x_1+\bar{\Qb}\bar{\g}.
\]
For $0\leq \gamma\leq 1/2$, Gaussian small ball guarantees
\[
\Pro( |g_1|\leq \gamma)\geq \frac{7\gamma}{10}.
\]
Next, we argue that $\z_i=\li{\bar{\Qb}\bar{\g}},\x_i\ri$ is small for all $i\neq 1$. For a fixed $i\geq 2$, observe that
\[
\z_i\sim\Nn(0,1-(\x_1^T\x_i)^2).
\]
Note that 
\[
1-|\x_1^T\x_i|=\frac{\min(\tn{\x_1-\x_i}^2,\tn{\x_1+\x_i}^2)}{2}\geq \frac{\delta^2}{2}.
\]
Hence $1-(\x_1^T\x_i)^2\geq \delta^2/2$. From Gaussian small ball and variance bound on $\z_i$, we have
\[
\Pro(|\z_i|\leq \gamma)\leq \sqrt{\frac{{2}}{{\pi}}}\frac{\gamma}{\sqrt{1-(\x_1^T\x_i)^2}}\leq\frac{2\gamma}{\delta\sqrt{\pi}}
\]
Union bounding, we find that, with probability $1-\frac{2n\gamma}{\sqrt{\pi}\delta}$, we have that, $|\z_i|> \gamma$ for all $i\geq 2$. Since $\bar{\g}$ is independent of $g_1$, setting $\gamma=\frac{\delta}{2\sqrt{2}n}$ (which is at most $1/2$ since $\delta\leq \sqrt{2}$), 
\[
\Pro(E):=\Pro(|g_1|\leq \gamma,~|\z_i|> \gamma~\forall~i\geq 2)\geq (1-\frac{2n\gamma}{\sqrt{\pi}\delta})\frac{2\gamma}{5}\geq \frac{\delta}{12n}.
\]
To proceed, note that
\[
f(\g):=\ab^T\ind{\X\w}=a_1\ind{g_1}+\sum_{i= 2}^n(a_i\times\ind{\x_i^T\x_1g_1+\x_i^T\bar{\Qb}\bar{\g}})
\]
On the event $E$, we have that $\ind{\x_i^T\x_1g_1+\x_i^T\bar{\Qb}\bar{\g}}=\ind{\x_i^T\bar{\Qb}\bar{\g}}$ since $|g_1|\leq \gamma\leq |\x_i^T\bar{\Qb}\bar{\g}|$. Hence, on $E$, 
\[
f(\g)=a_1\ind{g_1}+\text{rest}(\bar{\g}),
\]
where $\text{rest}(\bar{\g})=\sum_{i= 2}^n(a_i\times\ind{\x_i^T\bar{\Qb}\bar{\g}})$. Furthermore, conditioned on $E$, $g_1,\bar{\g}$ are independent as $\z_i$'s are function of $\bar{\g}$ alone hence, $E$ can be split into two equally likely events that are symmetric with respect to $g_1$ i.e.~$g_1\geq 0$ and $g_1<  0$. Consequently, 
\begin{align}
\Pro(|f(\g)|\geq\max(|a_1\ind{g_1}+\text{rest}(\bar{\g})|,|a_1\ind{-g_1}+\text{rest}(\bar{\g})|) \bgl E)\geq 1/2
\end{align}
Now, using $\max(|a|,|b|)\geq |a-b|/2$, we find
\[
\Pro(|f(\g)|\geq|a_1||\ind{g_1}-\ind{-g_1}|/2 \bgl E)=\Pro(|f(\g)|\geq|a_1| /2\bgl E)=\Pro(|f(\g)|\geq\tin{\ab} /2\bgl E)\geq 1/2.
\]
This yields $\Pro(|f(\g)|\geq\tin{\ab}/2)\geq \Pro(E)/2\geq \delta/24n$, concluding the proof by using $c_1=1/2,~c_2=1/25$.
\end{proof}
\begin{corollary}[Covariance of ReLU Jacobian] \label{sep cov}Let $\x_1,\dots,\x_n$ be points in $\R^d$ with unit Euclidian norm and $\w\sim\Nn(0,\Iden_d)$. Form the matrix $\X\in\R^{n\times d}=[\x_1~\dots~\x_n]^T$. Suppose there exists $\delta>0$ such that for every $1\leq i\neq j\leq n$, the input sample pairs have $\delta$ distance i.e.
\[
\min(\tn{\x_i-\x_j},\tn{\x_i+\x_j})\geq \delta.
\]
Then, using Lemma \ref{minHad} and Theorem \ref{sep thm main}
\begin{align}\label{require}
\E[\ind{\X\w}\ind{\X\w}^T \odot\X\X^T]\succeq \frac{\delta}{100n^2}.
\end{align}
\end{corollary}

\paragraph*{Proof of Theorem \ref{reluthmsep}}

\begin{proof} For proof, we wish to apply the Meta-Theorem \ref{reluthmg} with proper value of $\la(\X)$. Under Assumption \ref{ass sep}, using Corollary \ref{sep cov}, we have that
\[
\la(\X)\geq \frac{\delta}{100 n^2}.
\]
Substituting this $\la(\X)$ value results in the advertised result $k\geq \order{(1+\nu)^2 n^9\|\X\|^6/\delta^4}$ and the associated learning rate. 
\end{proof}

\end{document}